\documentclass[lettersize,journal]{IEEEtran}
\usepackage{amsmath,amssymb,amsfonts}
\usepackage{bbm}
\usepackage{mathtools}
\usepackage{upgreek}
\usepackage{cases}
\usepackage{xintexpr}
\usepackage{pifont}
\usepackage{ntheorem}
\usepackage{booktabs}
\usepackage{multirow}
\usepackage{graphicx}
\usepackage{subcaption}
\usepackage{overpic}
\usepackage{textcomp}
\usepackage{xcolor}
\usepackage{xifthen}
\usepackage{multicol}
\usepackage{cuted}
\usepackage{cite}
\usepackage{enumitem}
\usepackage[switch]{lineno}
\usepackage[numbers,sort&compress]{natbib}
\usepackage[ruled,linesnumbered,noend]{algorithm2e}
\usepackage{hyperref,soul}
\DeclareMathAlphabet{\mathpzc}{OT1}{pzc}{m}{it}

\def\BibTeX{{\rm B\kern-.05em{\sc i\kern-.025em b}\kern-.08em
    N\kern-.1667em\lower.7ex\hbox{E}\kern-.125emX}}

\newcommand{\cred}{\mathrm{CR}}

\newcommand{\de}[1][]{
\ifthenelse{\isempty{#1}}
{\Delta}{\Delta^{(#1)}}}
\newcommand{\accgrad}[1][]{
\ifthenelse{\isempty{#1}}
{\boldsymbol{\delta}}
{\boldsymbol{\delta}^{(#1)}}}
\newcommand{\hataccgrad}[1][]{
\ifthenelse{\isempty{#1}}
{\hat{\boldsymbol{\delta}}}
{\hat{\boldsymbol{\delta}}^{(#1)}}}

\newcommand{\lbound}{c_1}
\newcommand{\ubound}{c_2}

\newcommand*{\diff}{\mathop{}\!\mathrm{d}}

\newcommand{\diag}[1]{\mathrm{diag}\left(#1\right)} 

\newcommand{\dset}{\mathbb{D}}

\newcommand{\lw}[1][]{
\ifthenelse{\isempty{#1}}
{\boldsymbol{h}}
{\boldsymbol{h}^{(#1)}}}

\newcommand{\quanterr}[1][]{
\ifthenelse{\isempty{#1}}
{\boldsymbol{\zeta}}
{\boldsymbol{\zeta}^{(#1)}}}

\newcommand{\srQaunt}{Q_{\mathrm{sr}}}

\newcommand{\grad}[1][k]{\boldsymbol{g}^{(#1)}}
\newcommand{\stograd}[1][k]{\tilde{\boldsymbol{g}}_m^{(#1)}}

\newcommand{\x}{\mathbf{x}}
\newcommand{\y}{\mathbf{y}}

\newcommand{\vecIn}{\boldsymbol{a}}
\newcommand{\vecOut}{\hat{\boldsymbol{a}}}
\newcommand{\vecA}{\boldsymbol{a}}
\newcommand{\vecB}{\boldsymbol{b}}

\newcommand{\hw}[1][]{
\ifthenelse{\isempty{#1}}
{\mathbf{w}}
{\mathbf{w}^{(#1)}}}

\newcommand{\prob}{\mathbb{P}}

\newcommand{\qw}[1][]{\mathbf{w}}

\newcommand{\param}[1][]{
\ifthenelse{\isempty{#1}}
{\boldsymbol{\theta}}
{\boldsymbol{\theta}^{(#1)}}}

\newcommand{\eps}[1][]{
\ifthenelse{\isempty{#1}}
{\boldsymbol{\varepsilon}}
{\boldsymbol{\varepsilon}^{(#1)}}}
\newcommand{\epsidx}[1][]{
\ifthenelse{\isempty{#1}}
{\varepsilon^{(t)}}
{\varepsilon^{(t)}_{#1}}}

\newcommand{\err}[1][]{
\ifthenelse{\isempty{#1}}
{\boldsymbol{e}}
{\boldsymbol{e}^{(#1)}}}

\newcommand{\nw}[1][]{
\ifthenelse{\isempty{#1}}
{\widetilde{\mathbf{w}}}
{\widetilde{\mathbf{w}}^{(#1)}}}

\newcommand{\tw}[1][]{
\ifthenelse{\isempty{#1}}
{\widetilde{\mathbf{w}}}
{\widetilde{\mathbf{w}}^{(#1)}}}

\newcommand{\res}[1][]{  
\ifthenelse{\isempty{#1}}
{\mathbf{r}}
{\mathbf{r}^{(#1)}}}

\newcommand{\w}[1][]{
\ifthenelse{\isempty{#1}}
{\mathbf{w}}
{\mathbf{w}^{(#1)}}}

\newcommand{\argmin}{\mathop{\mathrm{argmin}}\limits} 

\newcommand{\ind}[1]{
    \mathbbm{1}\left( #1 \right)
}

\newcommand{\expect}[1][]{
\ifthenelse{\isempty{#1}}
{\mathbb{E}}
{\mathbb{E}\left[#1\right]}}
\newcommand{\expectXi}[1]{
{\mathbb{E}_{\xi}\big[#1 \big]}}

\newcommand{\fw}[1][]{
\ifthenelse{\isempty{#1}}
{\mathbf{\bar{w}}}
{\mathbf{\bar{w}}^{(#1)}}
}

\newcommand{\innerprod}[2]
{\left\langle #1,\,#2 \right\rangle}
\newcommand{\innprod}[2]
{\Big\langle #1,\,#2 \Big\rangle}
\newcommand{\inprod}[2]
{\big\langle #1,\,#2 \big\rangle}

\newcommand{\interval}{\mathbb{H}^{(r)}_{m,i}}

\newcommand{\normsq}[1]
{\big\| #1\big\|_2^2}
\newcommand{\normSQ}[1]
{\left\| #1\right\|_2^2}

\newcommand{\oriw}{\boldsymbol{\theta}}

\newcommand{\order}[1]{O\left( #1 \right)}

\newcommand{\posindex}{\mathcal{I}_{m}^{(k,t)}}
\newcommand{\negindex}{\mathcal{I}_{m}^{-(k,t)}}

\newcommand{\pdist}[1][m]{\mathcal{P}_{m}}

\newcommand{\p}[1][]{
\ifthenelse{\isempty{#1}}
{\boldsymbol{p}}
{\boldsymbol{p}^{(#1)}}
}

\newcommand{\srpr}{\boldsymbol{\pi}}
\newcommand{\srprob}{\pi}

\newcommand{\sign}{\operatorname{sign}}

\newcommand{\sumtau}{\sum_{t=0}^{\tau-1}}
\newcommand{\sumM}{\sum_{m=1}^{M}}
\newcommand{\sumAll}{\sumM\!\sumtau}

\newcommand{\cirone}
{\text{\ding{172}}}
\newcommand{\cirtwo}
{\text{\ding{173}}}
\newcommand{\cirthree}
{\text{\ding{174}}}
\newcommand{\cirfour}
{\text{\ding{175}}}

\newcommand\addpicture[3]{%
\setbox\mybox=\hbox{\includegraphics[scale=#3]{#2}}
\myboxwidth\wd\mybox    
\renewcommand\windowpagestuff{%
\includegraphics[scale=#3]{#2}
\captionof{figure}{A test figure.}}
\parpic[#1]{%
\begin{minipage}{\myboxwidth}
 \windowpagestuff 
\end{minipage} 
} }

\newenvironment{proof}[1][]{
\ifthenelse{\isempty{#1}}
{\par\vspace*{-1mm}\noindent\textit{Proof.} }
{\par\vspace*{-2mm}\noindent\textit{Proof of #1.} }}
{\hfill$\square$ \vspace*{2mm}}

\theoremstyle{mystyle} \newtheorem{Theorem}{Theorem}
\newtheorem{Lemma}{Lemma}

\newtheorem{Remark}{Remark}
\theoremstyle{mystyle} \newtheorem{Assumption}{Assumption}

\usepackage{tikz}

\newenvironment{dashedbox}[2]{
\par
\begin{tikzpicture}
\node[rectangle,minimum width=#1] (m) {\begin{minipage}{#1} #2 \end{minipage}};
\draw[dashed] (m.south west) rectangle (m.north east);
\end{tikzpicture}
}

\newcommand{\allAssmp}{\ref{assumption:lower_bound} to \ref{assumption:quant}}
\newcommand{\highlight}[1]{\vspace{1mm}\noindent{}\textbf{#1}\hspace{3mm}}
\newcommand{\papertheorem}[1]{Theorem~\ref{#1}}
\newcommand{\paperassumption}[1]{Assumption~\ref{#1}}
\newcommand{\paperlemma}[1]{Lemma~\ref{#1}}
\newcommand{\paperremark}[1]{Remark~\ref{#1}}
\newcommand{\paperfig}[1]{Fig.~\ref{#1}}
\newcommand{\papertab}[1]{TABLE~\ref{#1}}
\newcommand{\paperalg}[1]{Algorithm~\ref{#1}}
\newcommand{\paperappendix}[1]{Appendix~\ref{#1}}

\newcommand{\zoom}[2][0.8]{  
\scalebox{#1}{#2}
}

\definecolor{fedvote}{RGB}{235,255,251}
\definecolor{byzantinefedvote}{RGB}{235,246,255}

\begin{document}
\title{Federated Learning via Plurality Vote}
\author{Kai Yue, \textit{Graduate Student Member, IEEE}, Richeng Jin, \textit{Member, IEEE}, Chau-Wai Wong, \textit{Member, IEEE}, and Huaiyu Dai, \textit{Fellow, IEEE}}

\markboth{IEEE TRANSACTIONS ON NEURAL NETWORKS AND LEARNING SYSTEMS}%
{Communication-Efficient Federated Learning via Predictive Coding}

\maketitle
\makeatletter{\renewcommand*{\@makefnmark}{}
\footnotetext{
Manuscript received February 17, 2022; revised  September 28, 2022 and November 1, 2022; accepted November 22, 2022. 
This work was supported in part by the US National Science Foundation under grants CNS-1824518 and ECCS-2203214. (\emph{Corresponding author: Chau-Wai Wong.}) 

K. Yue, C.-W. Wong, and H. Dai are with the Department of Electrical and Computer Engineering, NC State University, Raleigh, NC 27695 USA (e-mail: \{kyue, chauwai.wong, hdai\}@ncsu.edu). 

R. Jin is with the College of Information Science and Electronic Engineering, Zhejiang University, the Zhejiang–Singapore Innovation and AI Joint Research Lab, and the Zhejiang Provincial Key Lab of Information Processing, Communication and Networking (IPCAN), Hangzhou, China, 310000  (e-mail: richengjin@zju.edu.cn). }
\makeatother}

\begin{abstract}
    Federated learning allows collaborative clients to solve a machine learning problem while preserving data privacy.
    Recent studies have tackled various challenges in federated learning,  
    but the joint optimization of communication overhead, learning reliability, and deployment efficiency is still an open problem. 
    To this end, we propose a new scheme named federated learning via plurality vote~(FedVote).
    In each communication round of FedVote, clients transmit binary or ternary weights to the server with low communication overhead. 
    The model parameters are aggregated via weighted voting to enhance the resilience against Byzantine attacks.
    When deployed for inference, the model with binary or ternary weights is resource-friendly to edge devices.
    Our results demonstrate that the proposed method can reduce quantization error and converges faster compared to the methods directly quantizing the model updates.
 \end{abstract}


\section{Introduction}

Federated learning enables multiple clients to solve a machine learning problem under the coordination of a central server~\citep{kairouz2019advances}. 
Throughout the training stage, client data will be kept locally and only model weights or model updates will be shared with the server. 
Federated averaging (FedAvg)~\citep{mcmahan2017communication} was proposed as a generic federated learning solution. 
Although FedAvg takes advantage of distributed client data while maintaining their privacy, it leaves the following two challenges unsolved.
First, transmitting high-dimensional vectors between a client and the server for multiple rounds can incur significant communication overhead. 
In the federated learning literature, quantization has been directly applied to raw gradients to lower the communication overhead~\cite{reisizadeh2020fedpaq, haddadpour2020fedcom}. 
However, a more sophisticated design of a quantization scheme may provide better trade-offs between communication efficiency and model accuracy.
Second, the aggregation rule in FedAvg is vulnerable to Byzantine attacks~\citep{blanchard2017machine}. 
Prior works tackled this issue by using robust statistics such as coordinate-wise median and geometric median in the aggregation step~\citep{blanchard2017machine, yin2018byzantine}. 
Another strategy is to detect and reject updates from malicious attackers~\citep{munoz2019byzantine, sattler2020byzantine}. 
The robustness of the algorithm is enhanced at the cost of additional computational and algorithmic complexity. 

In this paper, we propose a new method called federated learning via plurality vote (FedVote). 
On the client side, we optimize a neural network with a range normalization function applied to model weights. 
After local updating, binary/ternary weight vectors are obtained via stochastic rounding and sent to the server.
The global model is updated by a voting procedure, and the voting results are sent back to each client for further optimization in the next round. 
The contributions of the paper are summarized as follows.

\begin{enumerate}
    \item We present FedVote as a novel federated learning solution. By exploiting the quantized model optimization on the client side and aggregation on the server side, we jointly optimize the communication overhead, learning reliability, and deployment efficiency.
    \item We theoretically and experimentally verify the effectiveness of our FedVote design.  
    In bandwidth-limited scenarios, FedVote is particularly advantageous in simultaneously achieving a high compression ratio and good test accuracy.
    Given a fixed communication cost, FedVote improves model accuracy on the non-i.i.d. CIFAR-10 dataset by $5$--$10\%$, $15$--$20\%$, and $25$--$30\%$ compared to FedPAQ~\citep{reisizadeh2020fedpaq}, signSGD~\citep{bernstein2018signsgd}, and FedAvg, respectively. 
    \item We extend FedVote to incorporate reputation-based voting in cross-silo federated learning.  
    The proposed method, Byzantine-FedVote, exhibits much better resilience to Byzantine attacks in the presence of close to half attackers without incurring excessive computation compared with existing algorithms. 
\end{enumerate}


\section{Related Work \label{section:related_works}} 
Federated learning serves as a privacy-preserving framework where clients can collaborate without sharing raw data. 
To improve communication efficiency, various strategies have been proposed.
For example, Sattler et al.~\cite{sattler2019robust} combined quantization and sparsification tools to lower the communication overhead.
Bernstein et al.~\cite{bernstein2018signsgd} showed that sign-based gradient descent schemes have good convergence rate in the homogeneous data distribution scenario.
This approach is further extended to the heterogeneous data distribution setting~\cite{chen2020distributed, jin2020stochastic, safaryan2021stochastic}.
FedAvg~\citep{mcmahan2017communication} adopts a periodic averaging scheme and targets at reducing the number of communication rounds. 
The convergence properties of FedAvg have been verified under different analytical frameworks~\cite{huang2021fl, li2020convergence}.
Hybrid methods consider simultaneous local updates and accumulative gradient compression~\citep{reisizadeh2020fedpaq, haddadpour2020fedcom}. 
In parallel with the aforementioned studies, cryptographic solutions have been adopted to provide stronger privacy protection, such as secure aggregation~\cite{bonawitz2016practical} and homomorphic encryption~\cite{zhang2020batchcrypt}. 
These algorithms may not be easily deployed in federated learning due to their special setups and additional communication and computational overheads.  

Meanwhile, there exist many open challenges to federated learning that is resilient to Byzantine adversaries, where arbitrary outputs can be produced.   
Blanchard et al.~\cite{blanchard2017machine} showed that FedAvg cannot tolerate a single Byzantine attacker. 
They proposed to select the gradient from normal clients based on the similarity of local updates. 
The idea of detecting malicious attackers based on similarity scores has been further developed in follow-up works.
Cao et al.~\cite{cao2020fltrust} calculated the cosine similarity scores between the malicious gradient and reference gradient based on a public small public dataset. 
Likewise, Shejwalkar et al.~\cite{shejwalkar2021manipulating} assumed a known number of attackers and filtered out malicious updates by computing the similarity score between the gradient and the extracted eigenvector.  

Inspired from the work of model quantization~\cite{hubara2016binarized, shayer2018learning, gong2019differentiable, qin2020binary}, we improve communication efficiency by using stochastic binary/ternary weights.
Our study is orthogonal to existing algorithms that focus on gradient compression. 
Furthermore, we propose a reputation-based voting strategy for  FedVote that is shown to have good convergence performance in the Byzantine setting. 
Compared to prior defense work, we do not assume a known number of attackers~\cite{blanchard2017machine,shejwalkar2021manipulating} and do not require the availability of a well-distributed public dataset~\cite{cao2020fltrust}.


A few recent work has also conducted case studies of quantized neural network in federated learning.
Lin et al.~\cite{lin2020ensemble} reported the model accuracy of 1-bit quantized neural networks aggregated via ensemble distillation.
The aggregation becomes complicated due to the separate optimization stage of knowledge distillation.
In addition, their BNN optimization is not tailored to the federated learning setting. 
Hou et al.~\cite{hou2019analysis} theoretically analyzed the convergence of distributed quantized neural networks,  
which was later implemented in the application of intrusion detection~\citep{qin2020line}.
Compared to~\cite{hou2019analysis}, we do not assume a convex and twice differentiable objective function and bounded gradients.
Therefore, their analysis cannot be directly applied to our study.  
We take a different perspective from the existing work and demonstrate that model quantization can be a more effective design to improve communication efficiency compared to gradient quantization.

\section{Preliminaries}\label{section:preliminary}

Symbol conventions are as follows. 
We use $[N]$ to denote the set of integers $\{1, 2, \dots, N\}$. 
Bold lower cases of letters such as $\boldsymbol{v}_m$ represent column vectors, and $v_{m,i}$ is the $i$th entry in the vector. 
For a scalar function, it applies elementwise operation when a vector input is given. 
$\mathbf{1} = [1,\dots,1]^\top$ denotes a vector with all entries equal to $1$.  

\subsection{Federated Learning}
The goal of federated learning is to build a machine learning model based on the training data distributed among multiple clients. 
To facilitate the learning procedure, a server will coordinate the training without accessing raw data~\citep{kairouz2019advances}.
In a supervised learning scenario, let $\mathcal{D}_{m}  = \{(\x_{m,j}, \y_{m,j})\}_{j=1}^{n_m} $ denote the training dataset on the $m$th client, with the input $\x_{m,j} \in \mathbb{R}^{d_1}$ and the label $\y_{m,j} \in \mathbb{R}^{d_2}$ in each training pair. 
The local objective function $f_m$ with a model weight vector $\param \in \mathbb{R}^{d}$ is given by 
\begin{equation}\label{eq:local_objective_function}
    f_{m}(\param) \triangleq f_{m}(\param; \mathcal{D}_m )  =  \frac{1}{n_m} \sum_{j=1}^{n_m} \ell(\param; (\x_{m,j}, \y_{m,j}) ), 
\end{equation}
where $\ell$ is a loss function quantifying the error of model $\param$  predicting the label $\y_{m,j}$ for an input $\x_{m,j}$. 
A global objective function may be formulated as 
\begin{equation}\label{eq:global_objective_function}
    \min_{\param \in \mathbb{R}^{d}} f(\param) = \frac{1}{M} \sum_{m=1}^{M} f_{m}(\param).
\end{equation}

\subsection{Quantized Neural Networks}\label{subsection:bnn}
Consider a neural network $g$ with the weight vector $\oriw \in \mathbb{R}^{d}$. 
A forward pass for an input $\x \in \mathbb{R}^{d_1}$ and a prediction $\hat{\y} \in \mathbb{R}^{d_2}$ can be written as $ \hat{\y} = g(\oriw, \x)$. 
In quantized neural networks, the real-valued $\oriw$ is replaced by $\qw \in \mathbb{D}_n^{d}$, where $\dset_{n} $ is a discrete set with a number $n$ of quantization levels.   
For example, we may have $\mathbb{D}_2 = \{-1,\,1\}$ for a binary neural network. 
For a given training set $\{(\x_j, \y_j)\}_{j=1}^{N}$ and the loss function $\ell$, the goal is to find an optimal $\qw^{*}$ such that the averaged loss is minimized over a search space of quantized weight vectors:
\begin{equation}\label{eq:quantized_nn_problem}
    \qw^{\star} =  \argmin_{\qw \in \dset^{d}_n} \frac{1}{N} \sum_{j=1}^{N} \ell(\qw; (\x_{j}, \y_{j})).
\end{equation}
Prior studies tried to solve \eqref{eq:quantized_nn_problem} by optimizing a real-valued latent weight vector $\lw \in \mathbb{R}^{d}$.
The interpretations of the latent weight vary when viewed from different perspectives. 
Hubara et al.~\cite{hubara2016binarized} used the sign operation to quantize the latent weight into two levels during the forward pass. 
The binary weights can be viewed as an approximation of their latent real-valued counterparts. 
Shayer et al.~\cite{shayer2018learning} trained a stochastic binary neural network, and the normalized latent parameters are interpreted as the Bernoulli distribution parameter $\vartheta_{i}$: 
\begin{equation}
    \vartheta_i \triangleq \widehat{\prob}(w_{i} = 1) = S(h_i), \quad w_i \in \{0, 1\},  
\end{equation}
where $S: \mathbb{R} \rightarrow (0,1)$ is the sigmoid function. 
In the forward pass, instead of using the binary vector $\qw$, its expected value, 
\begin{equation}
    \nw_{\text{sto-bnn}} \triangleq \mathbb{E}[\qw] 
    = -1 \! \times \! [\mathbf{1} \! - \! S(\lw)] + 1 \! \times \! S(\lw) = 2 \, S(\lw) -  \mathbf{1},
\end{equation}
will participate in the actual convolution or matrix multiplication operations. 
In other words, the neural network function becomes $\hat{\y} = g(\nw_{\text{sto-bnn}}, \x)$. 
Likewise, Gong et al.~\cite{gong2019differentiable} normalized the latent weight but interpreted it from a non-probabilistic viewpoint.
They approximated the staircase quantization function with a differentiable soft quantization  (DSQ) function, i.e.,  
\begin{equation}
    \nw_{\text{dsq}} \triangleq \tanh(a \lw), 
\end{equation}
where $\tanh:\mathbb{R} \rightarrow (-1,1)$ is the hyperbolic tangent function, and $a$ controls the shape of the function. 
The neural network function thus becomes $\hat{\y} = g(\nw_{\text{dsq}}, \x)$. 
The binary weights are obtained through the sign operation.

\section{Proposed FedVote Algorithm}\label{section:proposed_fedvote}

Inspired by prior works, we first present the latent weight-based BNN training methods. 
We depict an example of a single-layer network in \paperfig{fig:bnn_forwardpass} and give the details as follows.
First, a real-valued vector $\lw \in \mathbb{R}^{d}$ is introduced and its range is restricted by using a differentiable and invertible normalization function $\varphi: \mathbb{R} \rightarrow (-1, 1)$, 
such as the error function $\mathrm{erf}(\cdot)$ or hyperbolic function $\tanh(\cdot)$.
The forward pass is then calculated with the normalized weight vector $\nw$. The procedure is described as: 
\begin{equation}
    \hat{\y} =  g(\nw, \x), \quad \nw  \triangleq  \varphi(\lw) \label{eq:modified_forwardpass}.
\end{equation}
Second, in the back propagation, the latent weight vector $\lw$ is updated with its gradient, $\lw[t+1] = \lw[t] - \eta \nabla_{\lw} \ell$.   
Finally, the normalized weight vector $\nw$ is mapped to the discrete space to approximate $\qw^{*}$ via thresholding or stochastic rounding.

\begin{figure}
    \begin{overpic}[width=\linewidth]{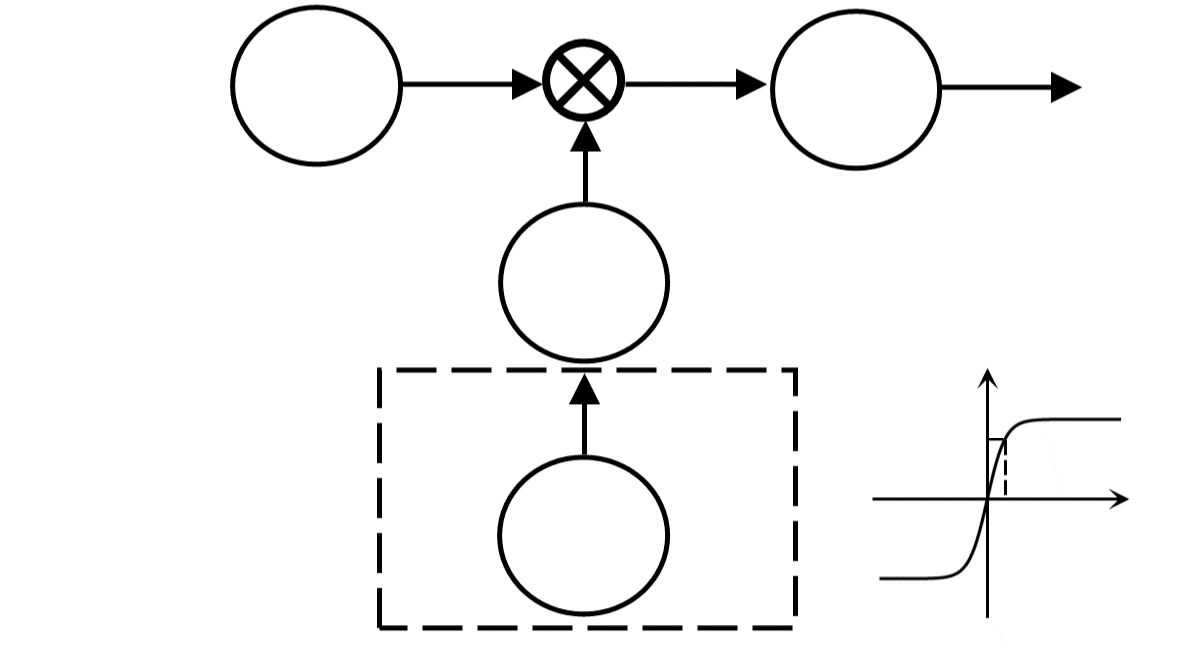}
        \put(5, 45){\zoom{input}}
        \put(25, 45){$\x$}
        
        \put(80, 40){$\sigma(\cdot)$}
        \put(68, 34.5){\zoom{nonlinear activation}}
    
        \put(67, 45){$\nw^{\top}\!\x $}
        \put(92, 45){$\hat{\y}$}
    
        \put(0, 18.){\zoom{normalization}}
        \put(34.5, 18.){$\varphi(\cdot)$}
    
        \put(0, 29){\zoom{normalized weight}}
        \put(47, 29){$\nw$}
    
        \put(0, 8){\zoom{latent weight}}
        \put(47, 8){$\lw$}
    
        \put(77, 22){$\varphi$}
        \put(78.5, 10.){$o$}
        \put(83.5, 9){$h_i$}
        \put(76, 16){$\widetilde{w}_i$}
    \end{overpic}
    \caption{
        Example of a single-layer quantized neural network with a latent weight vector $\lw \in \mathbb{R}^d$. 
        $\lw$ is normalized to generate $\nw \in (-1,1)^{d}$, and the output is $\hat{\y} = \sigma(\nw^\top \!\x) $.  
        A binary weight $\w$ can be obtained by thresholding or stochastically rounding $\nw$ to the discrete space $\mathbb{D}^d_2$. 
        \label{fig:bnn_forwardpass}  }    
\end{figure}
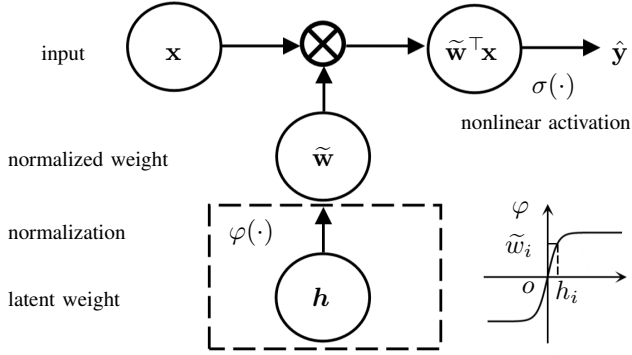

The separate representation of the quantized weight $\w$ and latent weight $\lw$ in the aforementioned model quantization scheme has introduced unique challenges in designing a federated training algorithm.  
For example, if clients choose to upload only binary weights $\w$ to save bandwidth, the latent weight $\lw$ on the server will be out of sync.
Besides, the advantages of designing model quantization methods over existing gradient quantization algorithms in federated learning is unexplored. 
To answer these questions, we present FedVote algorithm with an emphasis on uplink communication efficiency and enhanced Byzantine resilience.
We follow the widely adopted analysis framework in wireless communication to investigate only the client uplink overhead, assuming that the downlink bandwidth is much larger and the server will have enough transmission power~\citep{tran2019federated}.
To reduce the message size per round, we train a quantized neural network under the federated learning framework. 
The goal is to find a quantized weight vector $\hw^{*} $ that minimizes the global objective function $f$ formulated in (\ref{eq:global_objective_function}),
\begin{equation}\label{eq:fedvote_objective_function}
    \hw^{*} = \argmin_{\hw \in \dset^d_{n}} f(\hw).  
\end{equation}
For the simplicity of presentation,  we mainly focus on the BNN case with $\dset_{2} = \{-1,1\}$. 
We illustrate the procedure in \paperfig{fig:fedvote} and provide the pseudo code in Algorithm~\ref{alg:fedvote}.  
Below, we explain each step in more detail. 

\begin{algorithm}[tb]
    \caption{Binary-Weight FedVote with/without Byzantine Tolerance \label{alg:fedvote}}
    \SetKwComment{Comment}{$\triangleright$ }{}
    \SetKwBlock{server}{on server:}{}
    \SetKwBlock{worker}{on $m^{\mathrm{th}}$ worker:}{}
    \SetKwBlock{vote}{$\textrm{function } \text{vote}(\{\hw_{m}^{(k,\tau)}\}_{m=1}^{M})$}{}
    \SetKw{init}{initialize}
    \SetKw{rec}{receive}
    \SetKw{pull}{pull}
    \SetKw{from}{from}
    \SetKw{push}{push}
    \SetKw{send}{send}
    \SetKw{to}{to}
    \SetKw{broadcast}{broadcast}
    \DontPrintSemicolon
    
    \init $\p^{(0)}$ and \broadcast\\
    \For{$k=0,\, 1,\, \cdots, N-1$}{   
        \worker{
            \rec $\p[k]$ \from the server \\
            \init latent weight $\lw[k, 0]_m = \varphi^{-1}(2 \p[k] - 1)$   \\
            \For{$t = 0 : \tau-1$}{
                $\stograd[k, t] = \nabla_{\lw} f_{m}(\varphi(\lw[k,t]_{m}); \xi^{(t,r)}_{m} ) $\\
                $\lw[k,t+1]_{m} = \lw[k,t]_{m} - \eta^{(k,t)} \; \stograd[k,t] $
            }
            $\nw[k,\tau]_{m} = \varphi(\lw[k,\tau])$ \\
            $\hw[k,\tau]_{m} = \operatorname{sto\_rounding}(\nw[k,\tau]_{m})$ 
            \hspace*{8pt}
            $\triangleright $ \zoom[0.8]{\texttt{Eq. (\ref{eq:sto_rounding})}}\\
            \send $\hw[k,\tau]_{m}$ \to server 
        }
    
        \server{
            $\{\hw[k+1],\p[k+1]\}= \operatorname{vote}(\{\hw[k,\tau]_{m}\}_{m=1}^{M})$ \\
            \broadcast $\p[k+1]$ \to workers 
        }
        
    }
    \vspace{8pt}

    \vote{
        \For{$i=1:d$}{
            $w^{(k+1)}_{i} = \sign\left( \sum\limits_{m=1}^M w^{(k,\tau)}_{m} \right) $ \\[3pt]
            
            \dashedbox{0.395\textwidth}{
                $p_{i}^{(k+1)} = \frac{1}{M} \sum\limits_{m=1}^{M} 
                \ind{w_{m, i}^{(k,\tau)}=1}$ 
                \hspace*{8pt}
                $\triangleright $ \zoom[0.8]{\texttt{Option I}}
            }
            \vspace*{-8pt}
 
            \framebox[0.41\textwidth][l]{
                $p_{i}^{(k+1)} = \sum\limits_{m=1}^{M} \lambda^{(k)}_{m} \ind{w^{(k,\tau)}_{m,i}=1}$
                \hspace*{1pt}
                $\triangleright $ \zoom[0.8]{\texttt{Option II}}
            }
             
        }
        
    \vspace*{6pt}

        

        \Return{$\{ \hw[k+1], \p[k+1] \}$} to the server
    } 
\end{algorithm}
\subsection{Local Model Training and Transmission} 
We optimize a neural network with a learnable latent weight vector $\lw$.
In the $k$th communication round, we assume all clients are identically initialized by the server, namely, 
$\forall \; m \in [M]$, $\lw[k,0]_{m} = \lw[k]$. 
To reduce the total number of communication rounds, we first let each client conduct local updates to learn the binary weights. 
For each local iteration step, the local latent weight vector $\lw[k,t+1]$ is updated by the gradient descent:  
\begin{equation}\label{eq:lw_update_rule}
    \lw[k,t+1]_{m} = \lw[k,t]_{m} - \eta \; \nabla_{\lw} f_m \left(\varphi(\lw[k,t]_{m}); \xi^{(k,t)}_{m}  \right), 
\end{equation}
where $\xi_{m}^{(k,t)}$ is a mini-batch randomly drawn from $\mathcal{D}_m$ at the $t$th iteration of round $k$. 
After updating for $\tau$ steps, we obtain $\lw[k,\tau]_{m}$ and the corresponding normalized weight vector $\nw[k,\tau]_{m} \in (-1,1)^{d}$ defined as follows:  
\begin{equation}\label{eq:def_latent_weight}
    \nw[k,\tau]_{m} \triangleq \varphi(\lw[k,\tau]_m). 
\end{equation}
We use the stochastic rounding~\cite{alistarh2017qsgd} to draw a randomly quantized version $\hw[k,\tau]_{m}$ using $\nw[k,\tau]_{m}$, i.e.,
\begin{equation}\label{eq:sto_rounding}
    w^{(k,\tau)}_{m,i} = \left\{
        \begin{array}{l @{\;} l}
        +1,  &  \text {with prob. } \srprob^{(k,\tau)}_{m,i} = \frac{1}{2} \left[\widetilde{w}^{(k,\tau)}_{m,i} + 1\right], \\[3pt]
        -1, & \text {with prob. } 1 - \srprob^{(k,\tau)}_{m,i}. 
    \end{array}\right.
\end{equation}
It can be shown that the stochastic rounding is an unbiased procedure, i.e., $\mathbb{E} [\hw[k,\tau]_{m} | \nw[k,\tau]_{m}] = \nw[k,\tau]_{m}$. 
After quantization, the local client will send the binary weights $\hw[k,\tau]_{m}$ to the server for the global model aggregation. 

\begin{figure}
    \begin{overpic}[width=\linewidth]{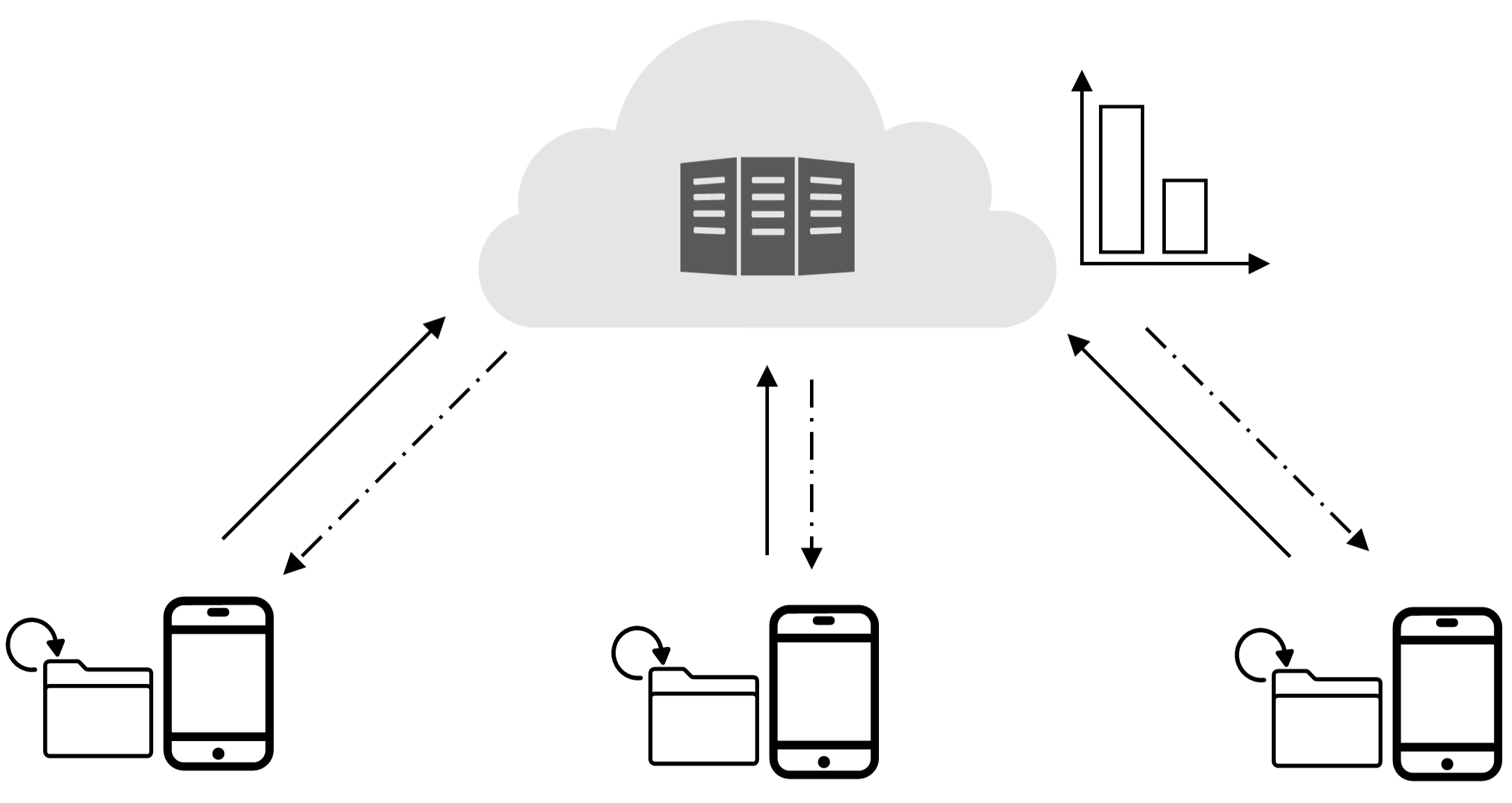}        
        \put(43,23){\zoom{\cirtwo}}
        \put(27,20){\zoom{sends quantized }}
        \put(27,16){\zoom{weight $\hw[k,\tau]_{m}$}}
        
        \put(37,11){\zoom{\cirone}}
        \put(17.5,7){\zoom{local updates }}
        \put(17.5,3){\zoom{that yields $\nw[k,\tau]_{m}$}}
        \put(43.5, 4){\zoom{$\mathcal{D}_{m}$}}
        \put(58, 7){\zoom{$m^{\text{th}}$}}
        \put(58, 4){\zoom{client}}
    
        \put(55, 22){\zoom{\cirfour}}
        \put(55, 19){\zoom{receives the }}
        \put(55, 16){\zoom{soft vote results}}
    
    
        \put(30, 45){\zoom{\cirthree}}
        \put(8,42){\zoom{server calculates the}}
        \put(8,39){\zoom{voting statistics}}
        \put(70.5, 33.5){\zoom[0.55]{$-1$}}
        \put(75.5, 33.5){\zoom[0.55]{$+1$}}
        \put(69, 37.5){\rotatebox{90}{\zoom[0.7]{count}}}

        \end{overpic}
        \caption{
            One round of FedVote is composed of four steps.
            Each client first updates the local model and then sends the quantized weight $\w[k,\tau]_m$ to the server. 
            Later, the server calculates the voting statistics and sends back the soft voting results $\p[k+1]$ to each client. 
            \label{fig:fedvote}  }
\end{figure}
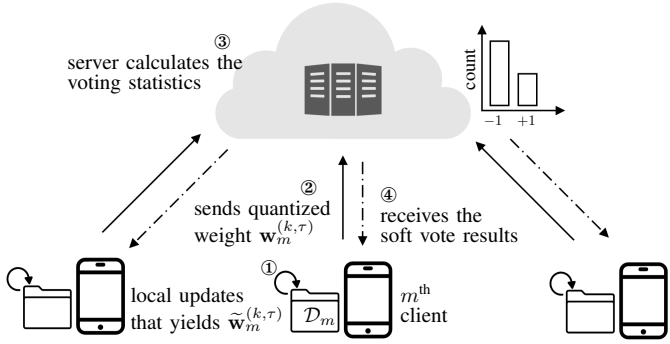

\subsection{Global Model Aggregation and Broadcast}
Once the server gathers the binary weights from all clients, it will perform the aggregation via plurality vote, i.e., $\hw[k+1] = \sign\left( \sum_{m=1}^M \hw[k,\tau]_{m} \right)$. 
A tie in vote will be broken randomly.
In the following lemma, we show that the probability of error reduces exponentially as the number of clients increases.
The proof can be found in \paperappendix{proof:one_shot_fedavg}.
\begin{Lemma}\label{lemma:one_shot_fedvote}
    (One-Shot FedVote) Let $\hw^{*} \in \mathbb{D}^{d}_{2}$ be the optimal solution defined in (\ref{eq:fedvote_objective_function}). 
    For the $m$th client, $\varepsilon_{m,i} \triangleq \prob(w^{(k, \tau)}_{m, i} \neq w_{i}^{*})$ denotes the error probability of the voting result of the $i$th coordinate.
    Suppose the error events $\{w^{(k, \tau)}_{m, i} \neq w_{i}^{*}\}_{m=1}^M$  are mutually independent, and the mean error probability $s_i = \frac{1}{M}\sum_{m=1}^M \varepsilon_{m,i}$ is smaller than $\frac{1}{2}$.
    For the voted weight $\w[k+1]$, we have
    \begin{equation}
        \prob\left(w^{(k+1)}_{i} \neq w_{i}^{*}\right) 
        \leqslant \big[2 s_i \exp(1 - 2 s_i )\big]^{\frac{M}{2}}.
    \end{equation}
\end{Lemma}

In practice, the number of available clients may be limited in each round, and the local data distribution is often heterogeneous or even time-variant. 
To extend the training to multiple communication rounds, we first use the soft voting to build an empirical distribution of global weight $\hw$, i.e., 
\begin{equation}\label{eq:vote_distribution}
    \widehat{\prob}(w^{(k+1)}_{i} = 1) 
    = \frac{1}{M} \sum_{m=1}^{M} \ind{w^{(k,\tau)}_{m,i}=1 },
\end{equation}
where $\ind{\cdot} \in \{0,\,1\}$ is the indicator function.
Let $p^{(k+1)}_{i} \triangleq \widehat{\prob}(w^{(k+1)}_{i} = 1)$ and 
$\p[k+1] = [p^{(k+1)}_{1}, \dots, p^{(k+1)}_{d}]^\top$.
The global latent parameters can be constructed by following~\eqref{eq:def_latent_weight}: 
\begin{equation}\label{eq:latent_weight_reconstruct}
    \lw[k+1] = \varphi^{-1}(2\, \p[k+1] - 1),
\end{equation} 
where $\varphi^{-1}:(-1,1) \rightarrow \mathbb{R}$ is the inverse of the normalization function $\varphi$. 
We further apply clipping to restrict the range of the probability, namely, $\operatorname{clip}(p^{(k+1)}_i) = \max(p_{\min}, \min(p_{\max}, p^{(k+1)}_i))$, 
where $p_{\min}, p_{\max} \in (0,1)$ are predefined thresholds. 
To keep the notation consistent, we denote $\nw[k+1] \triangleq \varphi(\lw[k+1])$ as the global normalized weight. 
After broadcasting the soft voting results $\p[k+1]$, all clients are synchronized with the same latent weight $\lw[k+1]$ and normalized weight $\nw[k+1]$. 
The learning procedure will repeat until a termination condition is satisfied.  
We relate FedVote to FedAvg in the following lemma. 
The detailed proof can be found in \paperappendix{proof:FedVote_recovers_FedAvg}.
\begin{Lemma}\label{lemma:FedVote_recovers_FedAvg}
(Relationship with FedAvg) For the normalized weights, FedVote recovers FedAvg in expectation: $\mathbb{E} \left[\nw[k+1] \right]  = \frac{1}{M} \sum_{m=1}^{M} \nw[k,\tau]_{m}$, 
where $\nw[k+1]  = \varphi(\lw[k+1])$ and $\nw[k,\tau]_{m} = \varphi(\lw[k,\tau]_{m})$.
\end{Lemma}

\subsection{Reputation-Based Byzantine-FedVote}
In this work, we assume a portion of the participants are Byzantine attackers that can access the data of other clients and modify transmitted message.
\paperlemma{lemma:FedVote_recovers_FedAvg} shows that FedVote is related to FedAvg in expectation. 
As we have reviewed in Section~\ref{section:related_works}, FedAvg cannot tolerate a single Byzantine attacker. 
It indicates that FedVote will exhibit similar poor performance in the presence of multiple adversaries (see Appendix~\ref{appendix:byzantine_fedvote}).
We improve the design  of FedVote based on a reputation voting mechanism, which is a variant of the weighted soft voting method. 

Reputation-based voting was presented in failure-robust large scale grids~\citep{bendahmane2014effectiveness}. 
Recent studies have also proposed reputation management schemes for federated learning~\cite{kang2020reliable, kang2019incentive}. 
However, assessing the reputation of different participants is nontrivial.
For example, the methods in~\cite{kang2020reliable, kang2019incentive} simplified the attacker model and require a reference to assess the reputation of clients.
Instead of building a reference explicitly, FedVote facilitates the reputation assessment by following the majority rule.
We modify~\eqref{eq:vote_distribution} to $\widehat{\prob}(w^{(k+1)}_{i} = 1) = \sum_{m=1}^{M} \lambda_{m}^{(k)} \ind{w^{(k,\tau)}_{m,i}=1}$, 
where $\lambda_m^{(k)}$ is proportional to a credibility score. 
In Byzantine-resilient FedVote (Byzantine-FedVote), we assume that at least $50\%$ of the clients behave normally and treat the plurality vote result as the correct decision. 
The credibility score of the $m$th client is calculated by counting the number of correct votes it makes: $\cred_{m}^{(k+1)} = \frac{1}{d} \sum_{i=1}^{d} \ind{w_{m, i}^{(k,\tau)} = w_{i}^{(k+1)}}$. 
Through multiple rounds, we track the credibility of a local client by taking an exponential moving average over the communication rounds, namely, 
$
    \nu^{(k+1)}_{m} = \beta \, \nu^{(k)}_{m} + (1 - \beta) \, \cred_{m}^{(k+1)},
$
where $\beta\in(0,1)$ is a predefined coefficient. 
The weight $\lambda_{m}^{(k)}$  is designed as
$
    \lambda_{m}^{(k)} = \nu^{(k)}_{m}/\sum_{m=1}^{M} \nu^{(k)}_{m}.
$

Although we focus on the binary weights in this section, the scheme can be naturally extended to quantized neural networks of more discrete levels. 
We will briefly discuss the implementation for ternary weights in Section \ref{section:experiments}.

\section{Analysis of Algorithm}\label{section:algorithm_analysis}
In this section, we present the theoretical analysis of FedVote when local data are independent and identically distributed (i.i.d.). 
The empirical results in the non-i.i.d. setting are discussed in Section \ref{section:experiments}. 
We consider nonconvex objective functions for neural network optimization. 
We use the temporal average of the gradient norm as an indicator of convergence, which is commonly adopted in the literature~\cite{reisizadeh2020fedpaq, haddadpour2020fedcom}. 
The algorithm converges to a stationary point when the gradient norm is sufficiently small.
We denote the stochastic local gradient by
$\stograd[k,t] \triangleq \nabla_{\lw} f_m(\varphi(\lw[k,t]_{m}); \xi_{m}^{(k,t)})$. 
The local true gradient and global true gradient will be denoted by $\grad[k,t]_m \triangleq \expectXi{\stograd[k,t]}$, $\grad[k] \triangleq \nabla_{\lw} f(\varphi(\lw[k]))$, respectively.
In addition, let $\quanterr[k]_{m} \triangleq \nw[k,\tau]_{m} - \hw[k,\tau]_{m}$ denote the error introduced by stochastic rounding. 
According to its unbiased property, $ \expect_{\srpr} [\quanterr[k]_{m}] = \boldsymbol{0}$. 
With the aforementioned notations, we state five assumptions for the convergence analysis. 

\subsection{Assumptions}
\begin{Assumption}\label{assumption:lower_bound}
    (Lower bound) $\forall \; \lw \in \mathbb{R}^{d}, \, \nw \in (-1,1)^{d}$, the objective function is 
    lower bounded by a constant $f^*$, i.e., $f(\nw) \geqslant f^{*} = \min_{\lw \in \mathbb{R}^{d}} f(\varphi(\lw))$. 
\end{Assumption}
\begin{Assumption}\label{assumption:Lsmooth}
    ($L$-smoothness) \hspace*{2pt} $\forall \; \nw_1, \nw_2 \in (-1,1)^d$, $m \in \{1,\dots, M\}$,  
    there exists some nonnegative $L$ such that $\|\nabla f_{m}(\nw_1) - \nabla f_{m}(\nw_2) \|_2 \leqslant L \, \|\nw_1 - \nw_2 \|_2.$
\end{Assumption}
Assumptions \ref{assumption:lower_bound} to \ref{assumption:Lsmooth} are common for necessary analysis~\citep{wang2018cooperative}. 
We limit the range of the normalized weight $\nw$ while in a typical setting there is no restriction to the model weight. 

\begin{Assumption}\label{assumption:normalization_function}
The normalization function $\varphi: \mathbb{R} \rightarrow (-1,1)$ is strictly increasing. 
In particular, we assume its first derivative is bounded for all $h^{(k,t)}_{m,i}$, i.e.,
$\frac{\diff  }{\diff h } \varphi(h^{(k,t)}_{m,i}) \in [\lbound, \ubound]$,
where $\lbound,\, \ubound$ are positive parameters independent of $k,\,t,\, m,$ and $i$.
\end{Assumption}
\paperassumption{assumption:normalization_function} is not difficult to satisfy in practice.
For example, let $\varphi(h) = \tanh(a h)$, we have $\varphi^{\prime}(h) = a\left[1 - \tanh^2(ah)\right] $, with $\ubound=a$. 
Note that $\varphi$ quickly saturates with a large $h$ in the local updating. 
On the other hand, the empirical Bernoulli parameter $p_i$ will be clipped for stability, which indicates that $h^{(k,t)}_{m,i}$ will be upper bounded by certain $h_{\textrm{B}}$. 
In this sense, we have $\lbound = a\left[1-\tanh^{2}(ah_{\textrm{B}})\right]$. 
The next two assumptions bound the variance of the stochastic gradient and quantization noise.  
\begin{Assumption}\label{assumption:sto_grad}
    The stochastic gradient has bounded variance, i.e., $\expect \normsq{ \grad[k,t]_m - \stograd[k,t] } \leqslant \sigma_{\varepsilon}^{2}$,  
    where $\sigma_{\varepsilon}^2$ is a fixed variance independent of $k$, $t$ and $m$.
\end{Assumption}
\begin{Assumption}\label{assumption:quant}
    The quantization error $\quanterr[k]_{m}$ has bounded variance, i.e., $\expect \normsq{\quanterr[k]_{m}} \leqslant  \sigma_{k}^{2}$, 
    where $\sigma_{k}^{2}$ is a fixed variance independent of $m$.
\end{Assumption}

Note that quantization error is affected by the quantizer type and the 
corresponding input. For example, if the normalization function $\varphi$ approximates the sign function very well, the stochastic quantization error will be close to zero. 
Formally, the upper bound $\sigma_{\zeta}^2$ can be viewed as a function of input dimension $d$, which we formulate in the following lemma. 
The proof is in \paperappendix{proof:sto_rounding_quantization_error}.
\begin{Lemma}\label{lemma:sto_rounding_quantization_error}
Suppose we have an input $\vecIn \in (-1,1)^{d}$ for the quantizer $\srQaunt$ defined in (\ref{eq:sto_rounding}), 
then the quantization error satisfies $\mathbb{E} \left[  \normsq{ \srQaunt (\vecIn) -  \vecIn} \big| \vecIn \right] = d - \normsq{\vecIn}$. 
\end{Lemma}

For existing algorithms quantizing the model update $\accgrad[k]_{m} \triangleq \param[k] -  \param[k,\tau]_m$, 
the quantizer has the property $\mathbb{E} \left[  \normsq{Q(\x) -  \x} \big| \x \right] \leqslant q \normsq{\x}$. 
With a fixed quantization step, $q$ increases when the input dimension $d$ increases \citep{basu2020qsparse}. 
We state the result for a widely-used quantizer, QSGD \citep{alistarh2017qsgd}, which has been adopted in FedPAQ, in the following lemma.
\begin{Lemma}\label{lemma:fedpaq_quantization_error}
    Suppose we have an input $\x \in \mathbb{R}^{d}$ for the QSGD  quantizer $Q$.
    In the coarse quantization scenario, the quantization error satisfies
    $
        \mathbb{E} \left[  \normsq{  Q(\x)-\x} \big| \x \right] = \order{d^{\frac{1}{2}}} \normsq{\x}.
    $
\end{Lemma}

\subsection{Convergence Analysis}\label{subsection:convergence_analysis}
We state the convergence results in the following theorem. The proof can be found in \paperappendix{proof:convergence_iid_setting}.
\begin{Theorem}\label{theorem:convergence_iid_setting}
    For FedVote under Assumptions \allAssmp, 
    let the learning rate $\eta = O\left((\frac{c_1}{c_2})^2 \frac{1}{L\tau \sqrt{K}}\right)$, 
    then after $K$ rounds of communication, we have
    \begin{align}\label{eq:convergence_bound}
        & \frac{1}{K} \sum_{k=0}^{K-1} \lbound^2 \expect \normsq{\nabla f(\nw[k])} 
            \leqslant  \frac{2 \left[f(\nw[0]) - f(\nw^{*}) \right] }{\eta\tau K} \nonumber \\
        &\qquad + \ubound^2 L \eta \left[ \frac{1}{M} +  \frac{\lbound^2 L\eta (\tau-1)}{2} \right] \sigma_{\varepsilon}^2 \nonumber \\
        &\qquad +  \frac{L}{\eta\tau K M} \sum_{k=0}^{K-1}\sigma_{k}^2  
        + \frac{2(\ubound^2 - \lbound^2)}{\tau MK} \sum_{k=0}^{K-1}\sum_{m=1}^{M} R^{(k)}_{m},  
    \end{align}
    where $R^{(k)}_{m}\triangleq - \sum\limits_{t=0}^{\tau-1}\!\sum\limits_{i\notin\posindex} \expect\left[ (\nabla f(\nw[k]))_i  (\nabla f(\nw[k,t]_{m}))_i \right]$ and $\posindex  \triangleq \left\{i \in [d] \big| \grad[k]_{i} \grad[k, t]_{m,i} \geqslant 0 \right\}$. 
\end{Theorem}
The choice of the learning rate $\eta$ is discussed in \paperappendix{proof:convergence_iid_setting}. 
We further comment on the result in Theorem~\ref{theorem:convergence_iid_setting} below. 

\begin{Remark}
When there is no normalization function and quantization, i.e., $\varphi(x) = x$ with $\lbound = \ubound =1$, $\sigma_{k}^2 = 0$,
\papertheorem{theorem:convergence_iid_setting} recovers the result obtained in \cite{wang2018cooperative}.
\end{Remark}

%
\begin{figure}
    \centering
    \captionsetup[subfigure]{aboveskip=1pt}
    \captionsetup{aboveskip=0pt, belowskip=-1pt}
    \subcaptionbox{\label{subfig:accumulative_grad_distribution}}[0.24\textwidth]{
        \includegraphics[width=0.24\textwidth]{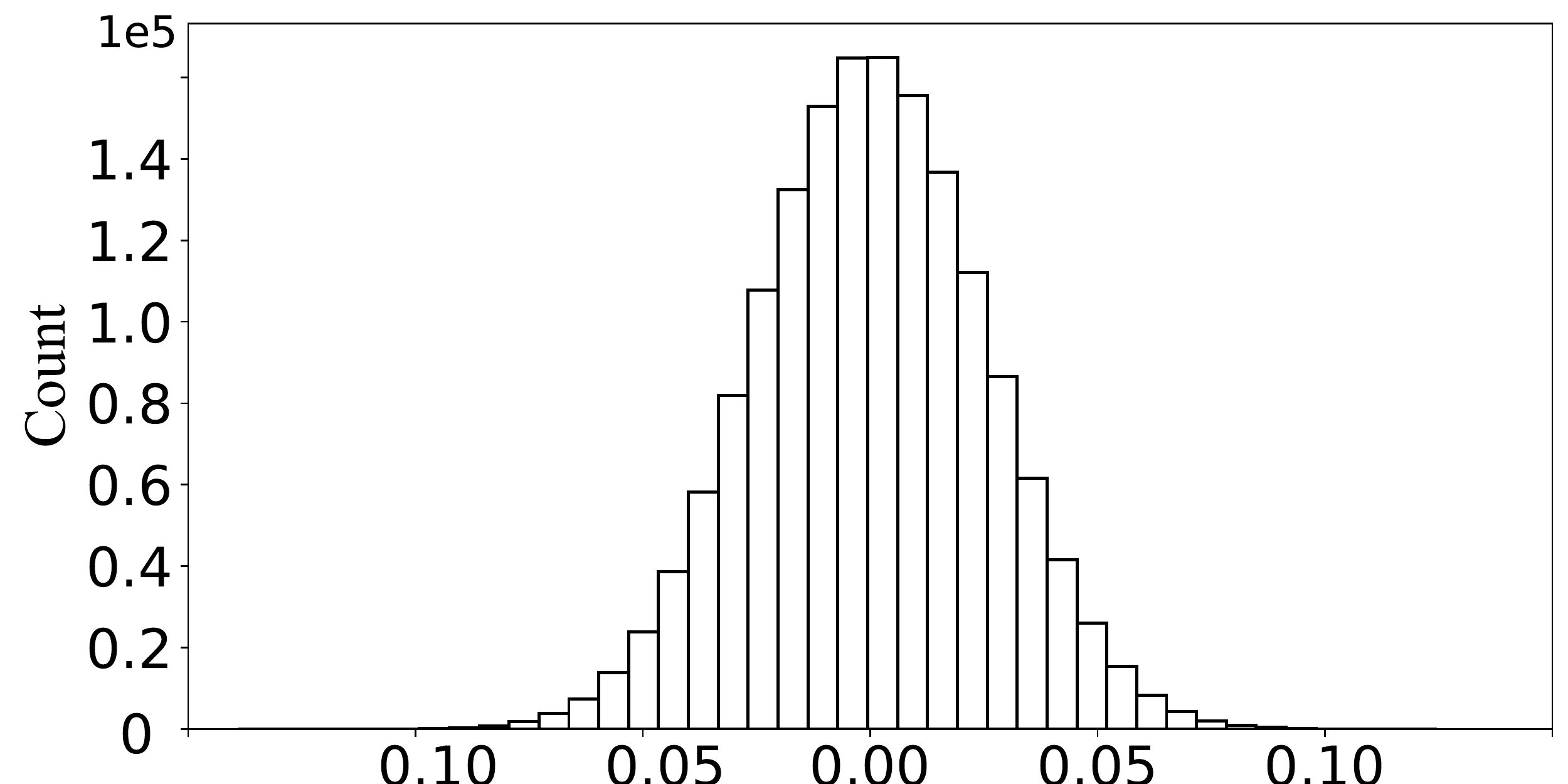}
    }
    \subcaptionbox{\label{subfig:latent_prob_distribution}}[0.24\textwidth]{
        \includegraphics[width=0.24\textwidth]{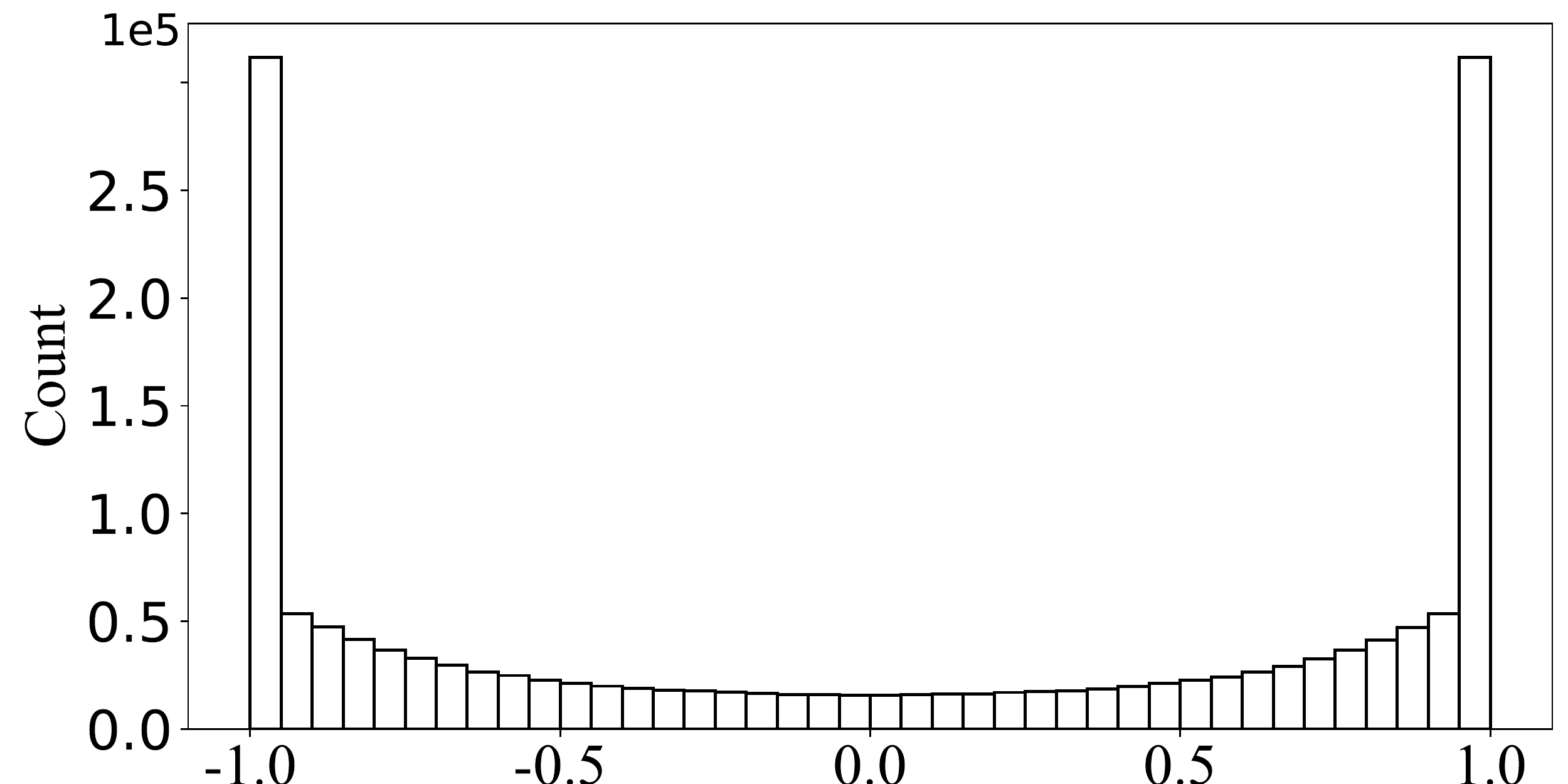}
    }\\[8pt]
    \caption{
    Histograms of (a) model updates $\delta^{(k)}_{m,i}$ and 
    (b) binary weight probabilities $\pi^{(k,\tau)}_{m,i}$. 
    We trained a LeNet model on the MNIST dataset for a single communication round and inspect the histograms on one client.   
    }
    \label{fig:quantizer_input_distribution}
\end{figure}
To discuss the impact of quantization error, consider the distribution of different inputs. 
For the model update $\accgrad[k]_{m}$, we expect the central limit theorem to render its distribution shape, where each entry $\delta^{(k)}_{m,i}$ follows the Gaussian distribution. 
For the Bernoulli probability $\srpr_m^{(k,\tau)}$, we expect the Beta distribution as the conjugate prior to render its distribution shape, where each entry $\srprob_{m,i}^{(k,\tau)}$ follows the symmetric Beta distribution.  
See \paperfig{fig:quantizer_input_distribution} for the empirical results. 

\begin{figure*}[!tb]
    \centering
    \subcaptionbox{}[0.32\textwidth]{
        \includegraphics[width=\linewidth]{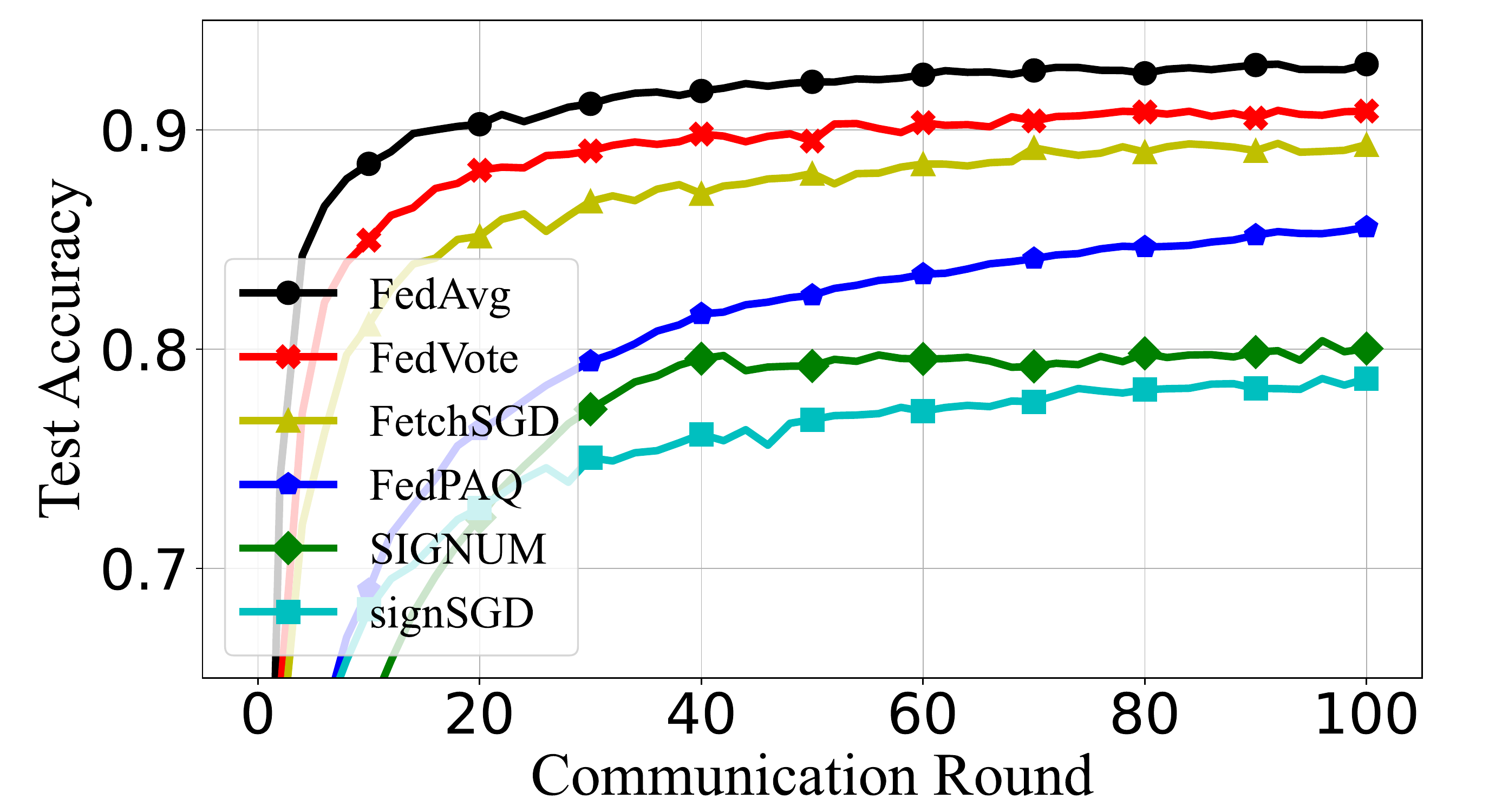}
    }
    \subcaptionbox{}[0.32\textwidth]{
        \includegraphics[width=\linewidth]{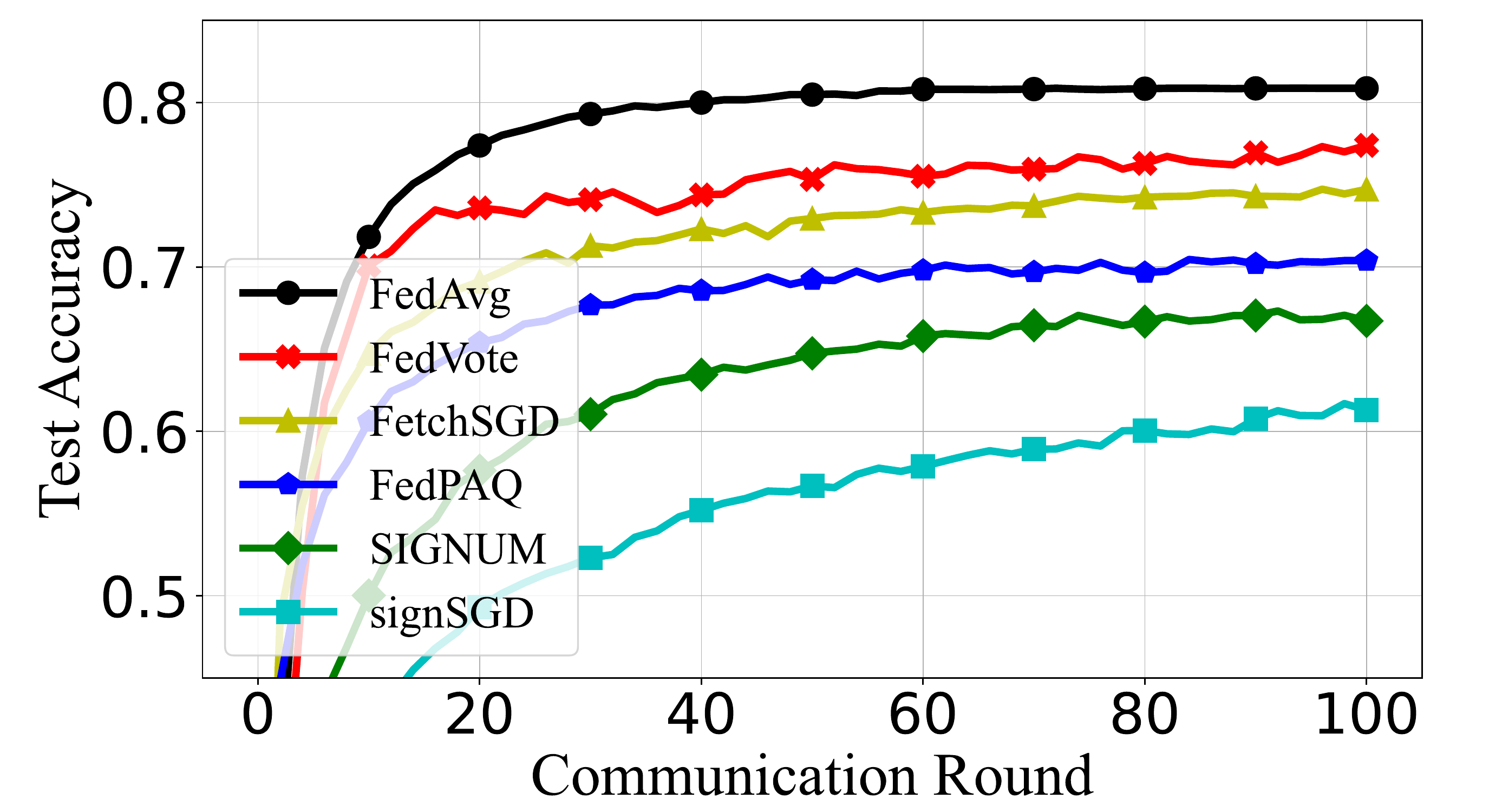}
    }
    \subcaptionbox{}[0.32\textwidth]{
        \includegraphics[width=\linewidth]{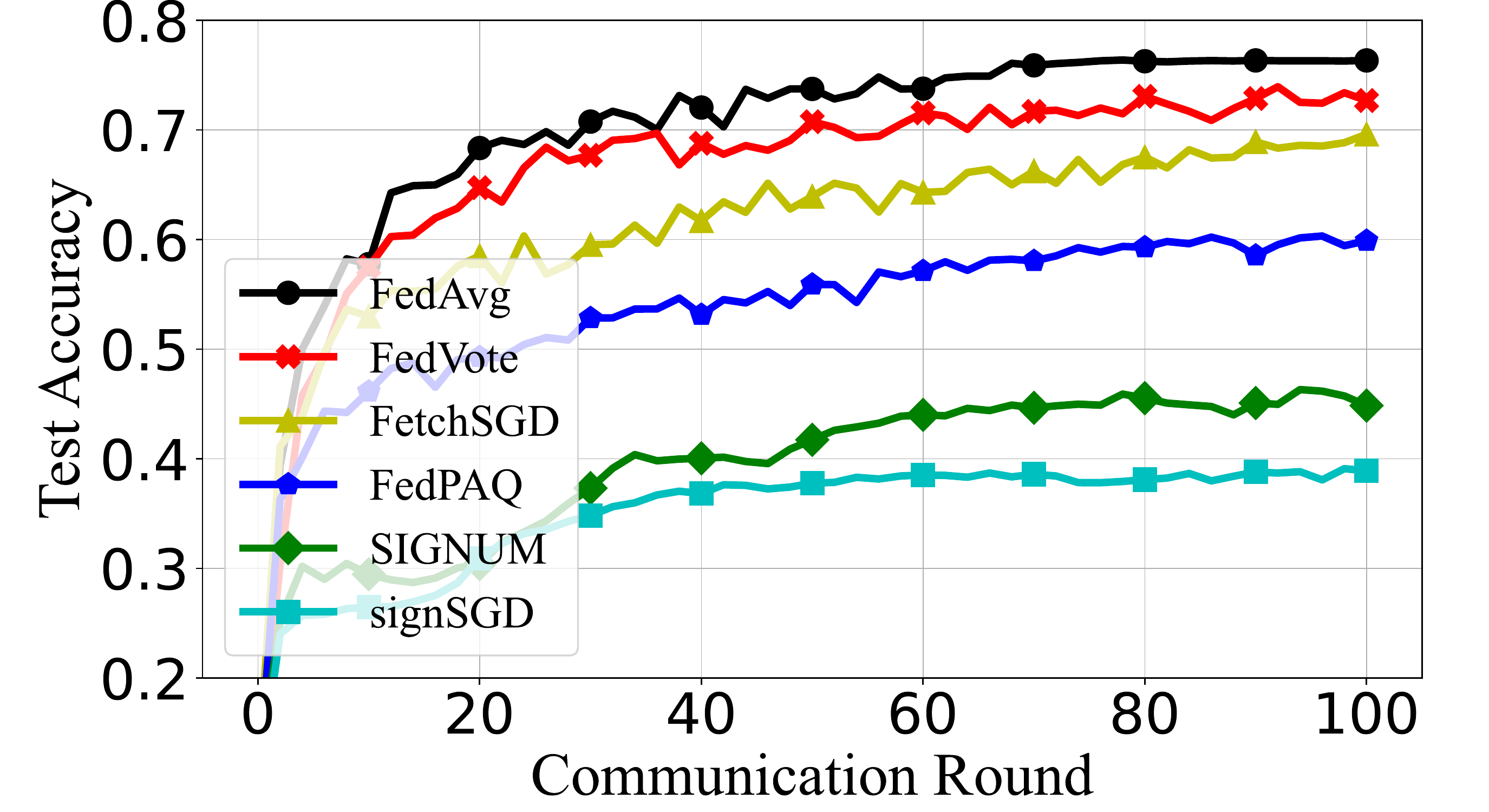}
    }
    \caption{
    Test accuracy versus communication round for different communication-efficient methods on 
    (a) cross-device non-i.i.d. FEMNIST with $300$ clients, 
    (b) cross-silo non-i.i.d. CIFAR-10 with $31$ clients, and 
    (c) cross-device non-i.i.d. CIFAR-10 with $100$ clients.
    FedVote outperforms other communication-efficient methods by achieving the highest accuracy given the communication round.}  
    \label{fig:lr_curves}
\end{figure*}

\begin{Remark}\label{remark:fedvote_outperforms_fedpaq}
Following the analysis framework in \papertheorem{theorem:convergence_iid_setting}, the order-wise convergence rate can be formulated as $\frac{1}{K} \sum_{k=0}^{K-1} \mathbb{E}\left\|\nabla f\left(\mathbf{w}^{(k)}\right)\right\|_{2}^{2}=\order{\frac{1}{\sqrt{K}}} + E(d)$, where $E(d)$ is the error introduced by the quantization. 
For FedVote with the normalized weight $\nw[k,\tau]_{m}$ as the input, when $\srprob_{m,i}^{(k,\tau)}$'s follow the symmetric Beta distribution, it can be shown that $\expect \|\quanterr[k]_{m}\|^2_2 = \order{d}$ based on Lemma \ref{lemma:sto_rounding_quantization_error}. 
For FedPAQ with the model update $\accgrad[k]_{m}$ as the input, when $\delta^{(k)}_{m,i}$'s follow the Gaussian distribution, it can be shown that $\expect \normsq{Q(\accgrad[k]_{m}) - \accgrad[k]_{m}} = \order{d^{3/2}}$ based on Lemma \ref{lemma:fedpaq_quantization_error}. 
When the weight dimension $d$ is sufficiently large, FedVote converges faster. 
\end{Remark}

\begin{Remark}
The value of the scalar error term $R^{(k)}_{m}$ in \eqref{eq:convergence_bound} depends on the gradient dissimilarity. 
If the angle between the local gradient $\nabla f_{m}(\nw[k,t]_m)$  and the global one $\nabla f(\nw[k,t])$ is not large, 
$R^{(k)}_{m}$ can be treated as a bounded variable.  
\end{Remark}

\begin{Remark}\label{remark:influence_normalization_function}
The choice of nonlinear function $\varphi:\mathbb{R}^{d} \rightarrow(-1,1)^{d}$ will affect the convergence. 
If $\varphi$ behaves more like the $\sign(\cdot)$ function, e.g., when $a$ increases in $\tanh(ax)$, the quantization error will be reduced. 
In other words, we expect a smaller $\sigma_{k}^2$ according to \paperlemma{lemma:sto_rounding_quantization_error}, 
which leads to a tighter bound in (\ref{eq:convergence_bound}).   
Meanwhile,  a larger $\ubound$ will negatively influence the convergence.   
\end{Remark}

\section{Experimental Results}\label{section:experiments}
\highlight{Data and Models.}
We choose image classification datasets, CIFAR-10~\citep{krizhevsky2009learning}, FEMNIST~\cite{caldas2018leaf}, and Fashion-MNIST~\citep{xiao2017fashion}. 
Suppose each dataset has $C$ label categories.
For FEMNIST, the heterogeneity comes from the unique writing styles of clients.
We consider two data partition strategies for CIFAR-10 and Fashion-MNIST: 
(i)~i.i.d. setting where the whole dataset is randomly shuffled and assigned to each worker without overlap; 
(ii)~non-i.i.d. setting where we follow~\cite{hsu2019measuring} and use a Dirichlet distribution to simulate the heterogeneity.  
In particular, for the $m$th worker we draw a random vector $\boldsymbol{q}_m \sim \text{Dir}(\alpha)$,
where $\boldsymbol{q}_m = [q_{m,1}, \cdots, q_{m,C}]^{\top}$ belongs to the standard $(C-1)$-simplex.  
We then assign data samples from different classes to the $m$th worker following the distribution of $\boldsymbol{q}_m$. 
We set $\alpha=0.5$ unless noted otherwise. 
We use a LeNet-5 architecture for Fashion-MNIST and a VGG-7 architecture for CIFAR-10 and FEMNIST. 
Results are obtained over three repetitions.

\highlight{Implementation Details.} 
We provide implementation details in the proposed FedVote design. 
First, following prior works~\citep{shayer2018learning, gong2019differentiable}, 
we keep the weights of the BNN final layer as floating-point values for the sake of the model performance. 
The weights of the final layer are randomly initialized with a shared seed and will be fixed during the training process. 
Second, we notice that for quantized neural networks, the batch normalization (BN)~\citep{ioffe2015batch} after the convolutional layer is necessary to scale the activation. 
We use the static BN without learnable parameters and local statistics~\citep{diao2021heterofl} to ensure the voting aggregation of binary weights. 
For the normalization function, we choose $\varphi(x) = \tanh(3x/2)$ unless noted otherwise.   
More details can be found in Appendix~\ref{appendix:hyperparameters}.

\begin{figure}[!tb]
    \centering
    \includegraphics[width=0.5\textwidth]{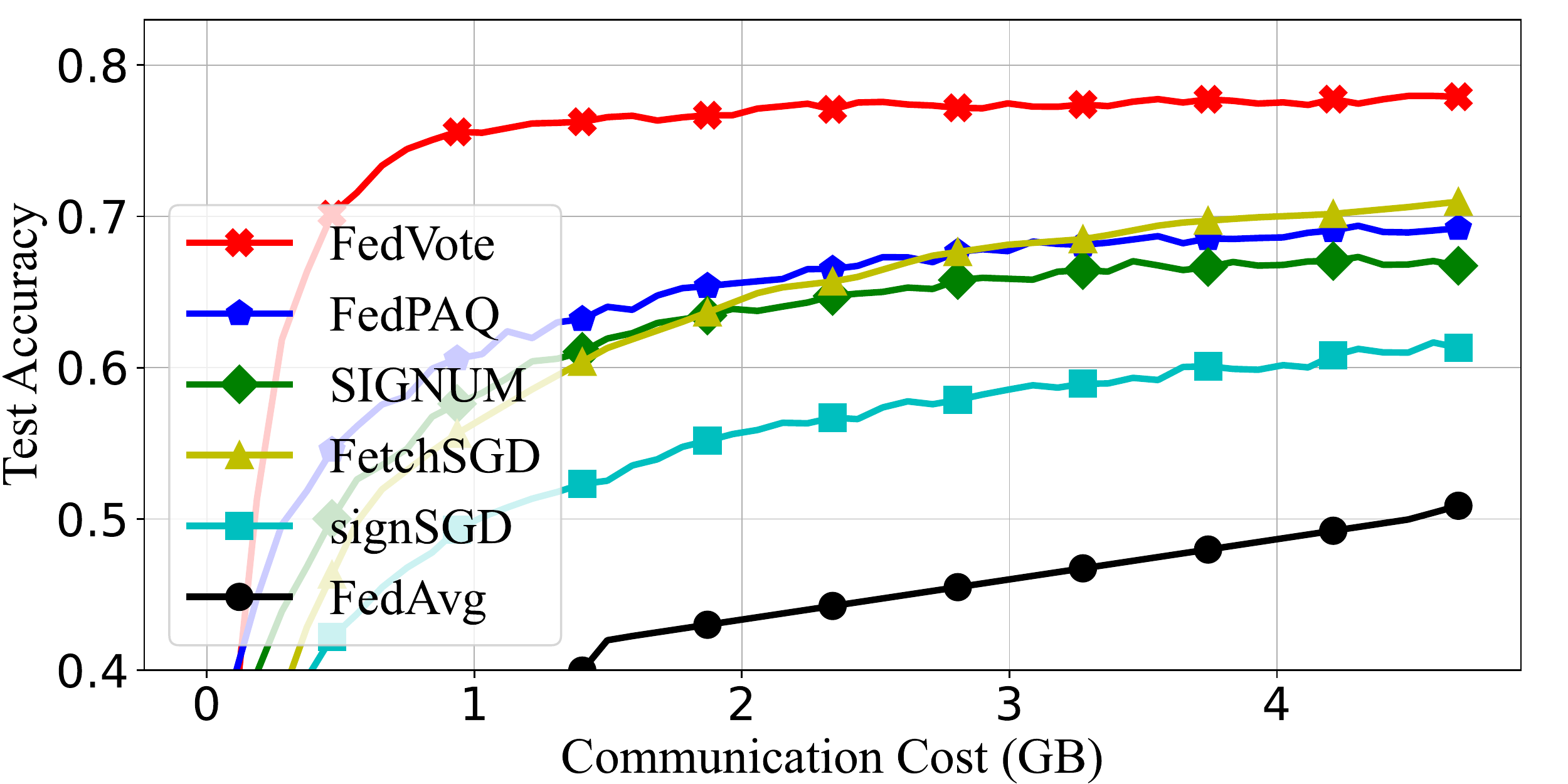}
    \caption{
    Test accuracy versus accumulative uplink communication cost on non-i.i.d. CIFAR-10. 
    FedVote outperforms other methods by achieving the highest test accuracy when the communication cost is fixed.} 
    \label{fig:acc_cost}
\end{figure}

\highlight{Communication Efficiency and Convergence Rate.} 
We simulate two cross-device settings and one cross-silo setting. 
For cross-device settings, we set $300$ clients in the FEMNIST task and $100$ clients in the CIFAR-10 task. 
We sample $20$ of them uniformly at random to simulate partial participation. 
For the cross-silo setting, we split CIFAR-10 training examples among $31$ clients. 
We compare FedVote with several popular gradient compression methods, including FetchSGD~\cite{rothchild2020fetchsgd}, FedPAQ~\cite{reisizadeh2020fedpaq}, signSGD~\cite{bernstein2018signsgd}, and SIGNUM~\cite{bernstein2019signsgd}. 
The results are shown in \paperfig{fig:lr_curves}.
The three plots reveal that FedVote outperforms gradient compression methods by achieving the highest accuracy in different settings. 
We plot test accuracy versus accumulative uplink communication cost on non-i.i.d. CIFAR-10 in \paperfig{fig:acc_cost}. 
FedVote outperforms the gradient quantization methods such as signSGD and SIGNUM that quantize gradients to 1 bit signs, and FedPAQ that quantizes the updates to 2 bits integers. 
It also shows advantage over the count sketch based scheme, FetchSGD, which adopted a special data structure for gradient compression. 
Compared with FedPAQ, signSGD, and FedAvg, FedVote improves the test accuracy by $5$--$10\%$, $15$--$20\%$, and $25$--$30\%$, respectively, given the fixed communication costs of $1.5$--$4.7$ GB.

\begin{figure*}[!tb]
    \centering
    \subcaptionbox{}[0.32\textwidth]{
        \includegraphics[width=\linewidth]{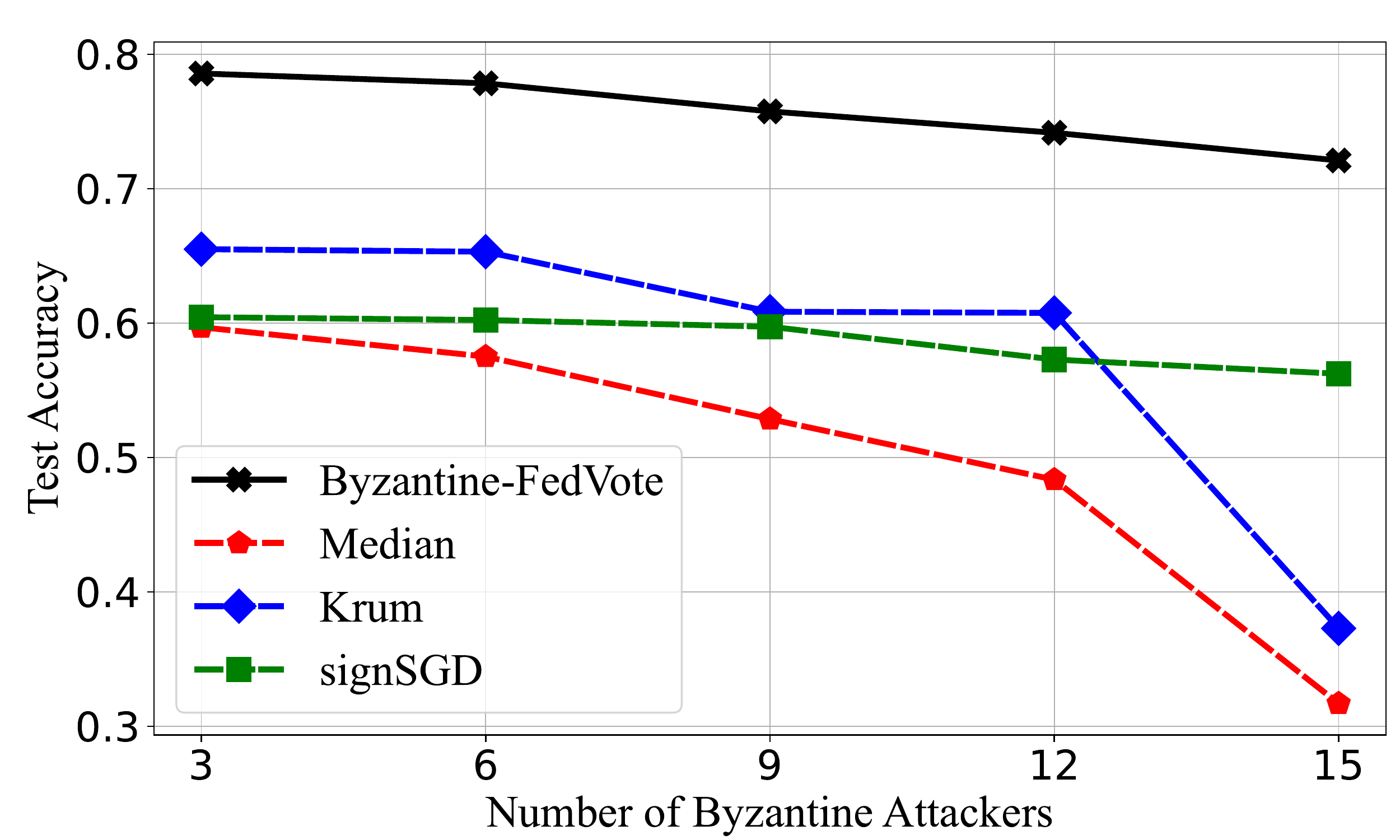}
    }
    \subcaptionbox{}[0.32\textwidth]{
        \includegraphics[width=\linewidth]{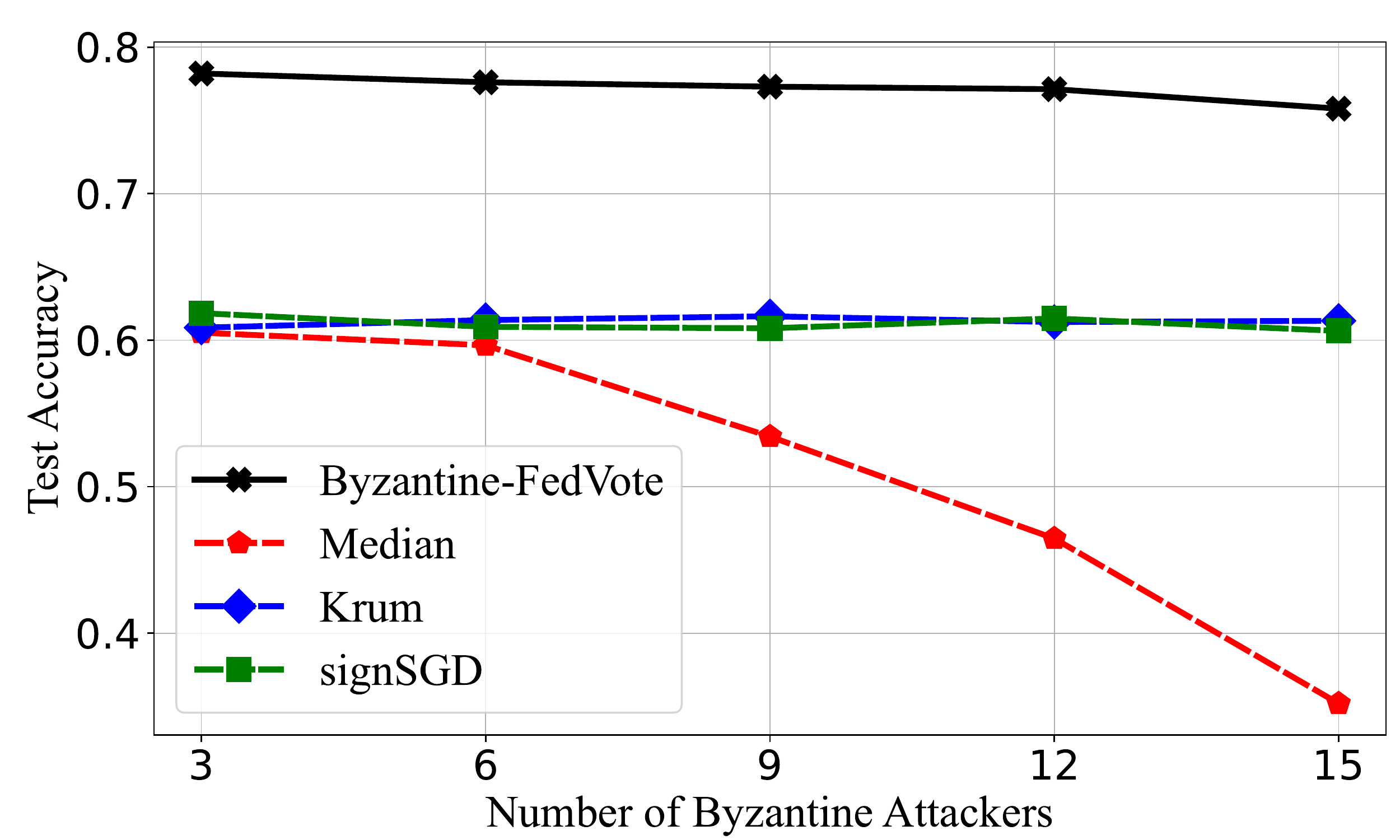}
    }
    \subcaptionbox{}[0.32\textwidth]{
        \includegraphics[width=\linewidth]{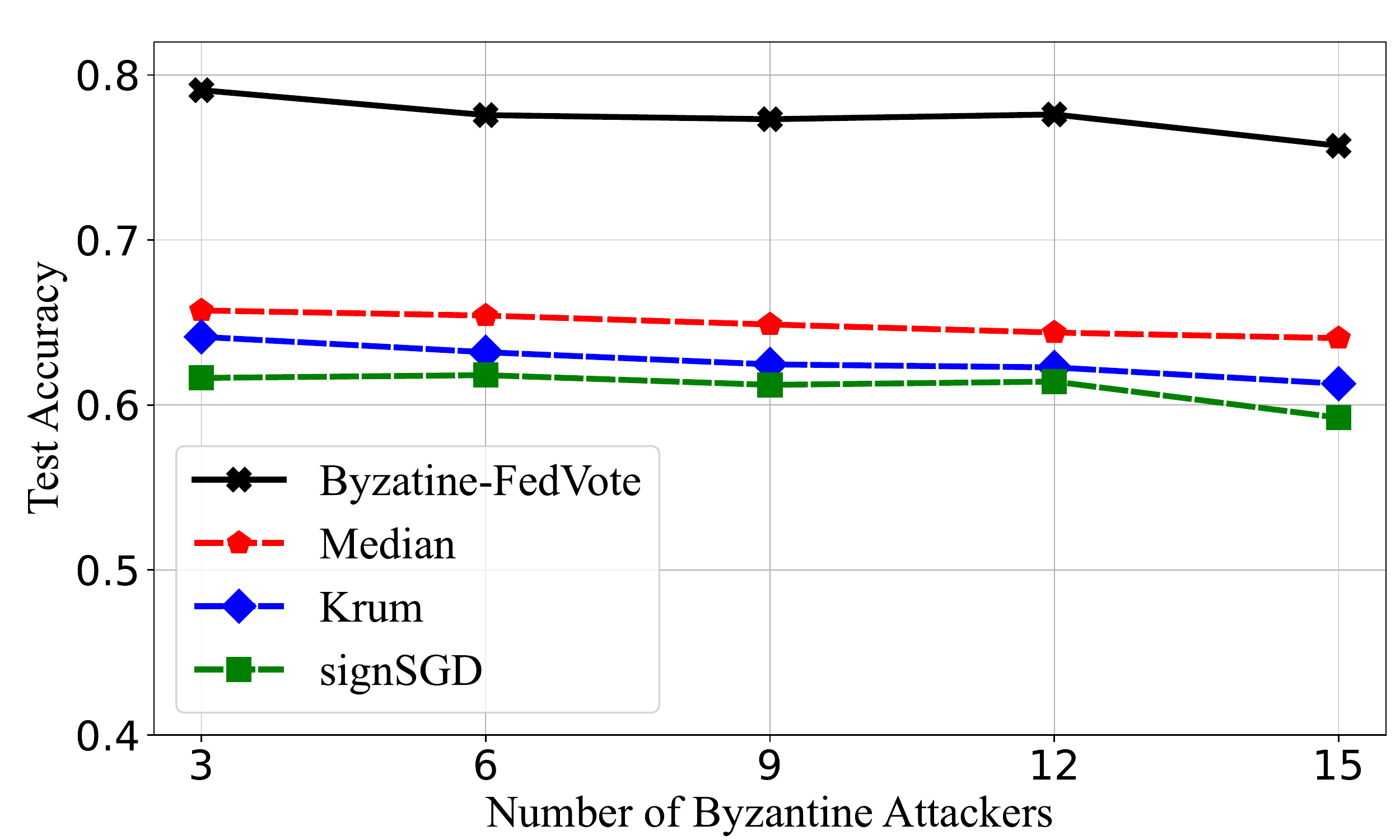}
    }
    \caption{
    Test accuracy versus the number of attackers on non-i.i.d. CIFAR-10 with 
    (a)~inverse sign attack, 
    (b)~data poisoning, and
    (c)~random perturbation.
    Byzantine-FedVote shows resilience under all types of attack, with less than $7\%$ drop in accuracy.
    \label{fig:acc_byz_r1}
    }
\end{figure*}

\highlight{Byzantine Resilience.} 
We consider the following three attack methods:
(i)~inverse sign, where attackers inverse the signs of transmitted weights/gradients;
(ii)~data poisoning~\cite{cao2020fltrust}, where attackers flip the labels of the training data;
(iii)~random perturbation, where attackers send random weights/gradients either from a binary uniform distribution or a Gaussian distribution, sharing the same statistics with normal clients. 
As Byzantine-FedVote needs to maintain the credibility score for each client, we use a cross-silo federated learning setting with full participation. 
The number of attackers is 15, and the remaining 16 clients are normal participants.
We compare the proposed method with coordinate-wise median based gradient descent~\citep{yin2018byzantine}, Krum~\cite{blanchard2017machine}, and signSGD~\citep{bernstein2018signsgd} in \paperfig{fig:acc_byz_r1}. 
Byzantine-FedVote achieves the best resilience among various resilient schemes under different attacks.

\highlight{Normalization Function.} From \paperremark{remark:influence_normalization_function}, we know that the normalization function can influence the model convergence. 
We empirically examine the impact in this experiment.
For normalization function $\varphi (x) = \tanh(a x)$, we choose $a$ from $\{0.5,1.5,2.5,10\}$.
We test the model accuracy after $20$ communication rounds on Fashion-MNIST.   
The results are shown in \papertab{table:effect_normalization_fun}. 
As $a$ increases, the linear region of the normalization function shrinks, and the algorithm converges slower due to a larger $c_2$. 
On the other hand, the gap between the model with normalized weight $\nw$ and the one with binary weight $\w$ also decreases due to smaller quantization errors. 

\begin{table}[tb]
    \centering
    \caption{\sc Effect of the Normalization Function}
    \label{table:effect_normalization_fun}
    \begin{tabular}{p{0.13\linewidth}p{0.1\linewidth}p{0.1\linewidth}p{0.1\linewidth}p{0.1\linewidth}p{0.1\linewidth}}
        \toprule
        \multirow{2}{0.13\linewidth}{Fashion-MNIST} & & \multicolumn{4}{c}{$a$} \\[-1.3pt] \cmidrule{3-6} 
        &  & $0.5$ & $1.5$ & $2.5$ & $10$ \\[-1.5pt] \midrule
        \multirow{2}{0.13\linewidth}{ i.i.d.} 
        &  float    & $90.7\%$ & $90.6\%$ & $90.0\%$ & $88.2\%$ \\
        & binary   & $88.7\%$ & $90.4\%$ & $89.9\%$ & $88.2\%$ \\ \cmidrule{1-6} 
        \multirow{2}{0.13\linewidth}{non-i.i.d.} 
        & float     & $87.3\%$ & $86.9\%$ & $85.7\%$ & $85.0\%$ \\
        & binary  & $83.3\%$ & $85.5\%$ & $85.2\%$ & $84.6\%$ \\ \bottomrule
    \end{tabular}  
\end{table}    

\highlight{Ternary Neural Network Extension.}
In the previous sections, we focus on the BNNs. 
We extend FedVote to ternary neural networks (TNNs) and empirically verify its performance. 
Training and transmitting the categorical distribution parameters of the ternary weight  may bring additional communication and computation cost to edge devices, we therefore simplify the procedure as follows.  
For each ternary weight $w^{(k,t)}_{m,i}$, we still keep a latent parameter $h^{(k,t)}_{m,i}$ to optimize. 
The detail of stochastic rounding can be found in Appendix~\ref{appendix:tnn}.
The training results are shown in \papertab{table:tnn_and_bnn}. 
As TNNs can further reduce the quantization error, their performance is better than the BNNs at the cost of additional $1$ bit/dimension communication overhead.

\begin{table}[tb]
\centering
\caption{\sc Test Accuracy of TNN and BNN \label{table:tnn_and_bnn}}
\begin{tabular}{p{0.19\linewidth}p{0.2\linewidth}p{0.16\linewidth}p{0.16\linewidth}}
    \toprule
    Dataset &  Distribution  & BNN & TNN \\ \midrule
    \multirow{2}{0.07\linewidth}{Fashion-MNIST} 
    &  i.i.d.   & $91.1\%$ & $91.9\%$ \\[2pt]
    & non-i.i.d & $88.3\%$ & $89.4\%$  \\ \cmidrule{1-4} 
    \multirow{2}{0.07\textwidth}{CIFAR-10} 
    & i.i.d.    & $80.5\%$ & $82.5\%$ \\[2pt]
    & non-i.i.d & $74.6\%$ & $77.6\%$ \\ \bottomrule
\end{tabular}
\end{table}
\begin{table}[tb]
    \centering
    \caption{\sc Forward Pass Efficiency}
    \label{table:deployment_efficiency}
    \begin{tabular}{p{0.14\linewidth}p{0.12\linewidth}p{0.14\linewidth}p{0.14\linewidth}p{0.11\linewidth}}
        \toprule
        Neural Net & Weight Type & Adds &  Muls & Energy (mJ) \\
        \midrule
        \multirow{2}*{LeNet-5} & float  & \zoom[0.9]{$1.7 \!\times\! 10^{9}$} & \zoom[0.9]{$1.8\!\times\! 10^{9}$}  & \zoom[0.9]{$8.1$} \\
                            & binary & \zoom[0.9]{$1.7 \!\times\! 10^{9}$} & \zoom[0.9]{$1.0 \!\times\! 10^{5}$}  & \zoom[0.9]{$1.5$} \\ \cmidrule{1-5}
        \multirow{2}*{VGG-7} & float  & \zoom[0.9]{$4.8 \!\times\! 10^{10}$} & \zoom[0.9]{$5.4 \!\times\! 10^{10}$}  & \zoom[0.9]{$242.9$} \\ 
                            & binary & \zoom[0.9]{$4.8 \!\times\! 10^{10}$} & \zoom[0.9]{$2.1\!\times\! 10^{5}$}  & \zoom[0.9]{$43.3$} \\
        \bottomrule
    \end{tabular}    
\end{table}

\highlight{Deployment Efficiency.} 
We highlight the advantages of BNNs during deployment on edge devices.
In FedVote, we intend to deploy lightweight quantized neural networks on the workers after the training procedure. 
BNNs require 32 $\!\times\!$ smaller memory size, which can save storage and energy consumption for memory access~\citep{hubara2016binarized}.  
As we do not quantize the activations, the advantage of BNNs inference mainly lies in the replacement of multiplications by summations.   
Consider the matrix multiplication in a neural network with an input $\x \in \mathbb{R}^{d_1}$ and output $\y \in\mathbb{R}^{d_2}$: 
$
    \y = W^\top \x.
$
For a floating-point weight matrix $W \in \mathbb{R}^{d_1 \!\times\! d_2}$, the number of multiplications is $d_1 d_2$, whereas for a binary matrix $W_{\textrm{b}} \in \mathbb{D}_2^{d_1 \!\times\! d_2}$ all multiplication operations can be replaced by an additions.   
We investigate the number of real multiplications and additions in the forward pass of different models and present the results in \papertab{table:deployment_efficiency}. 
We use the CIFAR-10 dataset and set the batch size to $100$. 
As to the energy consumption calculation, we use $3.7$ pJ and $0.9$ pJ as in~\cite{hubara2016binarized} for each floating-point multiplication and addition, respectively.

\section{Conclusion}\label{section:conclusion}
In this work, we have proposed FedVote to jointly optimize communication overhead, learning reliability, and deployment efficiency. 
In FedVote, the server aggregates neural networks with binary/ternary weights via voting.  
We have verified that FedVote can achieve good model accuracy even in coarse quantization settings. 
Compared with gradient quantization, model quantization is a more effective design that achieves better trade-offs between communication efficiency and model accuracy.
With the voting-based aggregation mechanism, FedVote enjoys the flexibility to incorporate various voting protocols to increase the resilience against Byzantine attacks.
We have demonstrated that Byzantine-FedVote exhibits much better Byzantine resilience in the presence of close to half attackers compared to the existing algorithms.

\setcounter{Lemma}{0}
\setcounter{Theorem}{0}

\appendices
\section{Setup and Additional Experiments}\label{appendix:setup}

\subsection{Hyperparameters}\label{appendix:hyperparameters}
For the clipping thresholds, we set $p_{\text{min}} = 0.001$ and $p_{\text{max}} = 1 - p_{\text{min}}$. 
The thresholds are introduced for numerical stability and have little impact on performance. 
We use $\beta = 0.5$ in Byzantine-FedVote. 
We use the Adam optimizer and search the learning rate $\eta$ over the set $\{10^{-4}, 3\times 10^{-4}, 10^{-3}, 3\times 10^{-3}, 10^{-2}, 3\times 10^{-2}, 10^{-1}, 3\times 10^{-1}\}$. 
We set the number of local iterations $\tau$ to 40 and the local batch size to $100$.
For FetchSGD~\cite{rothchild2020fetchsgd}, we set number of sketch columns to $1\times10^{6}$ and use $k=5\times 10^{4}$ in Top-$k$ method. 
Our implementation is available at \url{https://github.com/KAI-YUE/fedvote}.

\subsection{Comparison of Vanilla FedVote and Byzantine-FedVote}\label{appendix:byzantine_fedvote}
\paperlemma{lemma:FedVote_recovers_FedAvg} shows that FedVote is related to FedAvg in expectation. 
Adversaries sending the opposite results will negatively affect the estimation of the weight distribution and impede the convergence in multiple rounds. 
We compare the test accuracy of Byzantine-FedVote, Vanilla FedVote, and signSGD on the non-i.i.d. CIFAR-10 dataset with various numbers of omniscient attackers sending the opposite aggregation results. 
\paperfig{fig:acc_byz} reveals that the test accuracy of Vanilla FedVote drops severely when the number of adversaries increases, which is consistent with our analysis.   
In contrast, the drop of accuracy in Byzantine-FedVote is negligible. 

\subsection{Extension to Ternary Neural Networks} \label{appendix:tnn}
The stochastic rounding used in the ternary neural networks, $w_i = Q_{\text{sr}}(\widetilde{w})$, is an extension of \eqref{eq:sto_rounding}:
\begin{equation}
    w_i = \left\{
\begin{array}{ll}
   +1, & \textrm { with prob. } \pi_{1} = \tilde{w}_i \, \mathbbm{1}(\widetilde{w}_i > 0),   \\
   -1, & \textrm{ with prob. } \pi_{2} = - \tilde{w}_i \, \mathbbm{1}(\widetilde{w}_i < 0),  \\
   0, & \textrm{ with prob. } 1 - (\pi_{1} + \pi_{2}). 
\end{array}\right.
\end{equation}
One can modify the normalization function to optimize neural networks with multiple quantization levels. 

\begin{figure}
    \centering
    \includegraphics[width=0.5\textwidth]{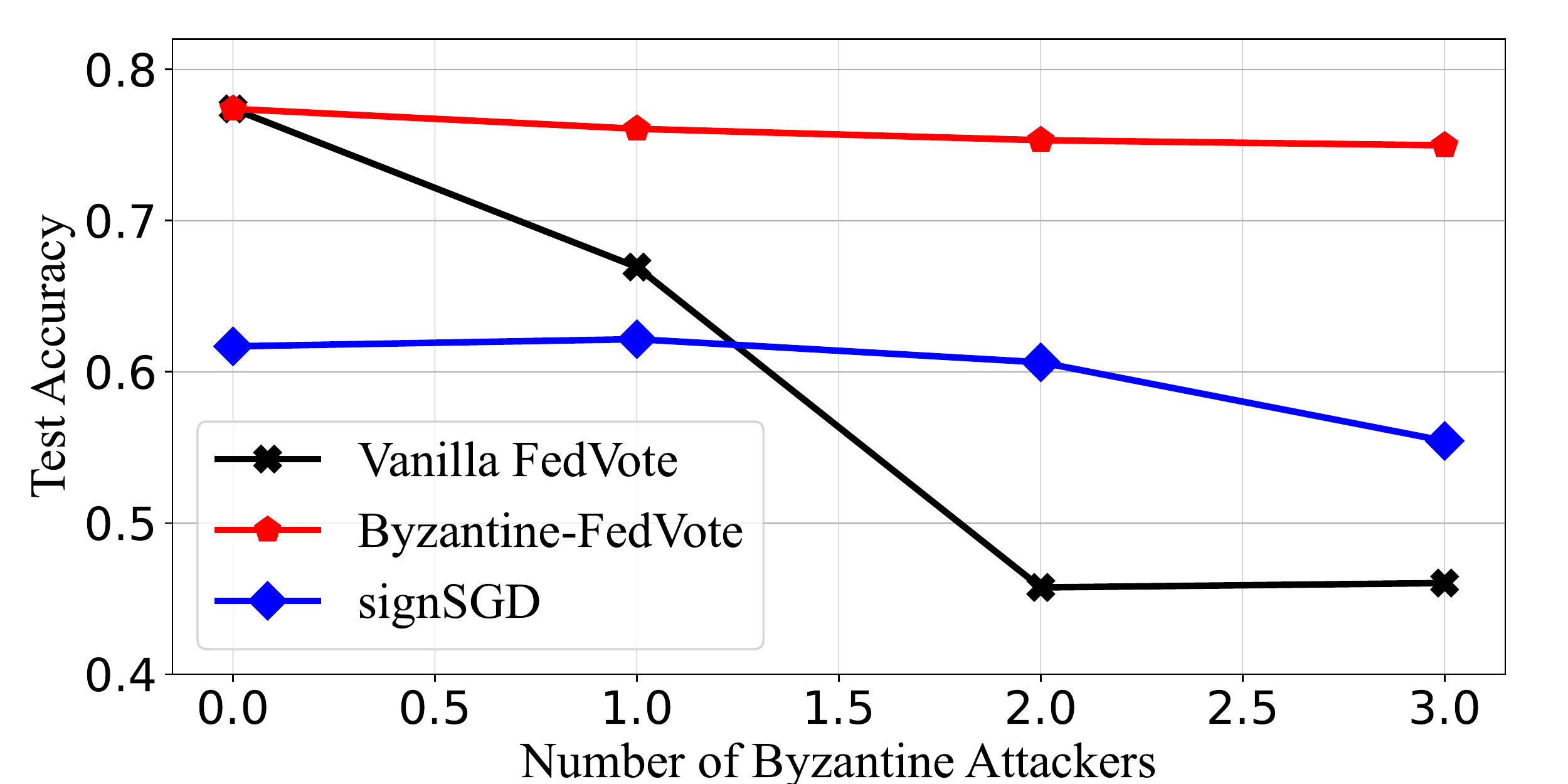}
    \caption{
        Test accuracy versus the number of Byzantine workers. 
            As the number of adversaries increases, the test accuracy of FedVote drops rapidly.
    \label{fig:acc_byz}
    }
\end{figure}

\subsection{Batch Normalization in FedVote}
Below we review the commonly-adopted BN function for convenience of presentation. 
For a one-dimensional input $x^{(j)}$ from the current batch $\mathcal{B}=\{x^{(1)},\cdots, x^{(n_{b})}\}$, the output of BN layer is formulated as 
\begin{equation}
    y \triangleq \text{BN}_{\gamma, b}(x^{(j)}) = 
    \gamma \frac{x^{(j)} - \mu}{\sqrt{\sigma^2 + \epsilon}} + b,
\end{equation}
where $\gamma, b$ are learnable affine transformation parameters, and 
$\mu, \sigma^2$ are the mean and variance calculated over the batch samples. 
Note that the normal BN layer will introduce the real-valued parameters 
and track the statistics of the input, all of which may cause problems when being binarized in FedVote. 
Therefore, we choose to set the parameter-free static BN, i.e.,
\begin{equation}
    y^{\prime} \triangleq \text{BN}(x^{(j)}) = 
    \frac{x - \expect_{\mathcal{B}}[ x^{(j)} ] }{\sqrt{ \text{Var}_{\mathcal{B}}[x^{(j)}] + \epsilon}}.
\end{equation}
\setcounter{Theorem}{0}
\setcounter{Lemma}{0}

\section{Missing Proofs}


\subsection{Proof of \paperlemma{lemma:one_shot_fedvote} \label{proof:one_shot_fedavg}}
\begin{proof}
Let $X_{m,i} \triangleq \ind{w^{(k, \tau)}_{m, i} \neq w_{i}^{*}}$, following Bernoulli distribution with parameter $\varepsilon_{m,i}$.
Let $Y_{i} = \sum_{m=1}^{M} X_{m,i}$, we have 
\begin{equation}
    \prob\left(w^{(k+1)}_{i} \neq w_{i}^{*}\right) 
    = \prob \left(Y_{i} \geqslant \frac{M}{2}\right).
\end{equation}
With independent vote results from workers, $Y_{i}$ follows Poisson binomial distribution 
with mean $\mu_{Y_{i}} = \sum_{m=1}^M \varepsilon_{m,i}$.
$\forall \; a>0$, the Chernoff bound can be derived as
\begin{subequations}
\begin{align}
    & \prob \left(Y_{i} \geqslant \frac{M}{2}\right) 
    = \prob(e^{aY_i}  \geqslant e^{\frac{aM}{2}} ) \\ 
    \qquad \overset{\cirone}{\leqslant} &\; \exp \left(-\frac{aM}{2}\right) \expect \left[e^{aY_i}\right] \\
    \qquad = &\; \exp \left(-\frac{a M}{2}\right) \prod_{m=1}^{M}\left(1-\varepsilon_{m,i}+e^{a} \varepsilon_{m,i}\right) \\
    \qquad = &\; \exp \left(-\frac{a M}{2} + \sum_{m=1}^{M} \ln \left( 1 + \varepsilon_{m,i}(e^a - 1) \right) \right) \\
    \overset{\cirtwo}{\leqslant} &\; \exp \left(-\frac{a M}{2} + \sum_{m=1}^{M} \varepsilon_{m,i}(e^a - 1) \right).
\end{align}
\end{subequations}
where $\cirone$ is based on Markov's inequality. $\cirtwo$ holds due to $\ln(1+x) \leqslant x$ for all $x\in(-1,\infty)$. 

By assumption we have $\mu_{Y_i} < \frac{M}{2}$. Let $a = \ln{\frac{M}{2\mu_{Y_i}}}$, we have
\begin{subequations}
\begin{align}
    \prob \left(Y_{i} \geqslant \frac{M}{2}\right) 
    & \leqslant \frac{\exp \left(-\mu_{Y_{i}}+\frac{M}{2}\right)}{\left(\frac{M}{2 \mu_{Y_i}}\right)^{\frac{M}{2}}} \\
    & \leqslant \left(\frac{2 \mu_{Y_{i}}}{M} \exp \left(1 - \frac{2\mu_{Y_{i}}}{M}\right)\right)^{\frac{M}{2}}. \label{eq:expaneded_one_shot_upperbound}
\end{align}  
\end{subequations}
Let $s_i = \frac{\mu_{Y_{i}}}{M} = \frac{1}{M} \sum_{m=1}^{M} \varepsilon_{m,i}$ and substitute it into (\ref{eq:expaneded_one_shot_upperbound}), 
the proof is complete.
\end{proof}

\subsection{Proof of \paperlemma{lemma:FedVote_recovers_FedAvg} \label{proof:FedVote_recovers_FedAvg}}
\setcounter{Lemma}{2}

\begin{proof}
From the inverse normalization in (\ref{eq:latent_weight_reconstruct}), we have $\nw[k+1] = 2\p[k+1] - 1$. 
Recall the definition of the empirical Bernoulli parameter $\p[k+1]$ given by (\ref{eq:vote_distribution}), we have the elementwise expectation
\begin{subequations}
\begin{align}
    \mathbb{E}_{\srpr} \left[\widetilde{w}^{(k+1)}_{i} \right]
    & = \frac{1}{M} \sum_{m=1}^{M}\left(2 \widehat{\prob}(w^{(k,\tau)}_{m,i} = 1) - 1 \right) \\
    & \overset{\cirone}{=} \frac{1}{M} \sum_{m=1}^{M} \varphi(h^{(k,\tau)}_{m,i}), \label{eq:final_step_fedvote_recovers_fedavg}
\end{align}
\end{subequations}
where $\cirone$ follows from stochastic rounding defined in (\ref{eq:sto_rounding}). 
Based on the definition of range normalization in (\ref{eq:def_latent_weight}), for local normalized weight we have $w^{(k,\tau)}_{m,i} = \varphi(h^{(k,\tau)}_{m,i})$. 
Substituting the result into (\ref{eq:final_step_fedvote_recovers_fedavg}) we have
\begin{equation}
    \mathbb{E}_{\srpr} \left[\nw[k+1] \right]  = \frac{1}{M} \sum_{m=1}^{M} \nw[k,\tau]_{m},
\end{equation}
which completes the proof. 
\end{proof}

\subsection{Proof of \paperlemma{lemma:sto_rounding_quantization_error} \label{proof:sto_rounding_quantization_error}}
    
\begin{proof}
Let $\vecOut \triangleq \srQaunt (\vecIn)$, we have 
\begin{subequations}
\begin{align}
    & \mathbb{E}\left[\normsq{  \srQaunt(\vecIn) -  \vecIn } \big| \vecIn\right] 
     = \mathbb{E}\left[\sum_{i=1}^{d} (\hat{a}_{i} - a_{i})^2 \big| \vecIn\right] \\ 
    \qquad = &\; \sum_{i=1}^{d}\left(\mathbb{E} \left[\hat{a}_{i}^{2} \big| \vecIn\right]-a_{i}^{2}\right)
     =  d-\|\vecIn\|_{2}^{2}.
\end{align}
\end{subequations}
The proof is complete. 
\end{proof}

\subsection{Proof of \paperlemma{lemma:fedpaq_quantization_error} \label{proof:fedpaq_quantization_error}}
\setcounter{Lemma}{4}

\begin{proof}
Consider a QSGD quantizer \citep{alistarh2017qsgd} with $s=1$. 
For detailed results of other quantizers, we refer readers to \cite{basu2020qsparse}. 

In particular, we have 
\begin{equation}
    Q_{\mathrm{qsgd}}\left(x_{i}\right)=\|\x\|_{2} \cdot \operatorname{sgn}\left(x_{i}\right) \cdot \xi_{i}(\x, s),
\end{equation}
where 
\begin{equation}
    \xi_{i}(\x, s)|_{s=1} =\left\{\begin{array}{ll}
    0 & \text { with prob. } 1-\frac{\left|x_{i}\right|}{\|\x\|_{2}}, \\
    1 & \text { with prob. } \frac{\left|x_{i}\right|}{\|\x\|_{2}}.
    \end{array}\right.
\end{equation}
The variance of quantization error is 
\begin{subequations}
\begin{align}
    & \mathbb{E}\left[ \normsq{Q(\x)-\x} | \x\right] 
    =\mathbb{E}\left[\sum_{i=1}^{d}\!\left( \hat{x}_{i}\!-\!x_{i}\right)^{2} \big| \x \right] \\
    \qquad = &\; \sum\limits_{i=1}^{d}\left(\mathbb{E}\left[\hat{x}_{i}^{2} \big| \x \right]- x_{i}^{2}\right) 
    = \frac{\|\x\|_{2}^{2}}{\|\x\|} \sum_{i=1}^{d}\left|x_{i}\right|\!-\!\|\x\|_{2}^{2} \\ 
    \qquad = &\; \|\x\|_{2}\|\x\|_{1}-\|\x\|_{2}^{2}
    \leqslant (\sqrt{d} - 1)\|\x\|_{2}^{2},
\end{align}
\end{subequations}
which completes the proof. 
\end{proof}

\subsection{Proof of \papertheorem{theorem:convergence_iid_setting} \label{proof:convergence_iid_setting}}
We first introduce some notations for simplicity.
Let $\de[k]$ denote the difference between two successive global latent weights, i.e., 
\begin{equation}\label{eq:def_nw_delta}
    \de[k] \triangleq \nw[k] - \nw[k+1].
\end{equation}  
We use $\eps[k,t]_m$ to denote the stochastic gradient noise, i.e., 
\begin{equation}
    \eps[k,t] \triangleq \stograd[k,t] - \grad[k,t]_m. 
\end{equation}
Finally, we let $\nabla f(\nw)$ denote the gradient with respect to $\nw$. 
The following five lemmas are presented to facilitate the proof. 
\begin{Lemma}\label{lemma:global_nw_aggregation}
The global normalized weight $\nw[k]$ can be reconstructed as 
the average of local binary weight, i.e.,
\begin{equation}\label{eq:nw_simple_agg}
\nw[k+1] = \frac{1}{M} \sum_{m=1}^{M} \hw[k,\tau]_{m}.
\end{equation}      
\end{Lemma}
\begin{proof}
See \paperappendix{proof:global_nw_aggregation}.
\end{proof}

\begin{Lemma}\label{lemma:lipschitz_continuity}
(Lipschitz continuity) Under \paperassumption{assumption:normalization_function}, 
$\forall\; x_1, \, x_2\in \mathbb{R}$, we have
\begin{equation}\label{lipschitz_continuity}
    |\varphi(x_1) - \varphi(x_2)| \leqslant \ubound |x_1 - x_2|. 
\end{equation} 
\end{Lemma}
\begin{proof}
Without loss of generality, suppose $x_1<x_2$. 
From the mean value theorem, there exists some $ c \in (x_1, x_2)$ such that 
\begin{equation}
    \varphi^{\prime}(c) = \frac{\varphi(x_2) - \varphi(x_1)}{x_2 - x_1}.
\end{equation}
For the monotonically increasing function $\varphi$, we have 
\begin{subequations}
\begin{align}
    \varphi\left(x_{2}\right)-\varphi\left(x_{1}\right) 
    &= \varphi^{\prime}(c)\left(x_{2}-x_{1}\right) \\
    & \overset{\cirone}{\leqslant} \ubound \left(x_{2}-x_{1}\right),
\end{align}
where $\cirone$ holds due to \paperassumption{assumption:normalization_function}. 
The similar result can be obtained by assuming $x_2<x_1$, which completes the proof.
\end{subequations}
\end{proof}

\begin{Lemma}\label{lemma:bound_w_divergence}
    (Bounded weight divergence) Under \paperassumption{assumption:sto_grad}, we have 
    \begin{equation}
        \expect \normsq{\nw[k,\tau]_{m}\!-\!\nw[k]}
        \leqslant 
        (\ubound \eta)^2 \tau
        \left( \sumtau \expect \normsq{\grad[k,t]_m} 
        + \sigma_{\varepsilon}^{2} \right). 
    \end{equation}
\end{Lemma}
\begin{proof}
    See \paperappendix{proof:bound_w_divergence}.
\end{proof}

\begin{Lemma}\label{lemma:bound_grad_divergence}
Under Assumptions \ref{assumption:Lsmooth} to \ref{assumption:quant}, we have
\begin{align}
    & \expect \innprod{\nabla f(\nw[k])}{ \de^{(t)} } 
    \geqslant 
    \frac{\lbound^2 \eta \tau}{2} \expect \normsq{\nabla f(\nw[k])}  \\
    &\quad + \frac{\lbound^2\eta}{4M} \left( 2 - (\ubound L)^2\eta^2\tau(\tau-1) \right)\sumAll \expect  \normsq{\grad[k,t]_m}  \nonumber \\
    &\quad - \frac{(\lbound \ubound L)^2\eta^3 \tau(\tau-1)}{4} \sigma_{\varepsilon}^2 
    - \frac{\eta (\ubound^2 - \lbound^2)}{M} \sum_{m=1}^M R^{(k)}_{m}, \label{eq:bounded_innerproduct}
\end{align}
where $R^{(k)}_{m}\triangleq - \sum\limits_{t=0}^{\tau-1}\!\sum\limits_{i\notin\posindex} \expect\left[ (\nabla f(\nw[k]))_i  (\nabla f(\nw[k,t]_{m}))_i \right]$ and $\posindex  \triangleq \left\{i \in [d] \big| \grad[k]_{i} \grad[k, t]_{m,i} \geqslant 0 \right\}$. 

\end{Lemma}
\begin{proof}
    See \paperappendix{proof:bound_grad_divergence}.
\end{proof}

\begin{Lemma}\label{lemma:bound_globalw_diff}
    Under Assumptions \ref{assumption:Lsmooth} to \ref{assumption:quant}, we have
    \begin{align}
        & \expect  \normsq{\de^{(k)}} 
        \leqslant
        \frac{(\ubound\eta)^2 \tau}{M} \sumAll \expect \normsq{\grad[k,t]_m} \nonumber \\
        & \quad + \frac{(\ubound\eta)^2\tau}{M} \sigma_{\varepsilon}^2 + \frac{1}{M}\sigma_{k}^2 . 
    \end{align}
\end{Lemma}
\begin{proof}
    See \paperappendix{proof:bound_globalw_diff}. 
\end{proof}

We give the proof of Theorem~\ref{theorem:convergence_iid_setting} as follows. \vspace*{1ex}

\begin{proof}
Consider the difference vector $\de[k]$ defined in (\ref{eq:def_nw_delta}), we expand it as 
\begin{subequations}
\begin{align}
    \de[k] & = \nw[k]-\nw[k+1] \\
    & \overset{\cirone}{=} \nw[k] - \frac{1}{M} \sum_{m=1}^{M} \hw[k,\tau]_{m} \\
    & \overset{\cirtwo}{=} \nw[k] - \frac{1}{M} \sum_{m=1}^{M} \nw[k,\tau]_{m} + \frac{1}{M} \sum_{m=1}^{M} \quanterr[r]_{m},
\end{align}
\end{subequations}
where $\cirone$ follows from (\ref{eq:nw_simple_agg}), 
and $\cirtwo$ holds by substituting the definition of quantization error.
From \paperassumption{assumption:Lsmooth}, we have
\begin{subequations}
\begin{align}
    f(\nw[k+1])\!-\!f(\nw[k]) \leqslant \!-\innprod{\nabla f(\nw[k])}{\!\de[k]\!} \!+\! \frac{L}{2} \normsq{\de[k]}.
\end{align}
\end{subequations}
Let $\frac{(L \eta)^{2}\tau( \tau-1)}{2} +  \frac{L\eta \tau}{\lbound^2} \leqslant \frac{1}{\ubound^2}$,  
take the expectation on both sides, and use Lemmas~\ref{lemma:bound_grad_divergence}--\ref{lemma:bound_globalw_diff}:
\begin{subequations}
\begin{align}
    &\; \expect \left[ f(\nw[k+1]) - f(\nw[k])\right] \leqslant 
    - \frac{\lbound^2 \eta \tau}{2} \expect \normsq{\nabla f(\nw[k])} - \nonumber  \\
    &\; \frac{(\lbound\ubound)^2\eta}{4M}\!\left(\frac{2}{\ubound^2} 
        \!-\!L^2\eta^2\tau(\tau-1)\!-\!\frac{2L\eta\tau}{\lbound^2} \right) \!\sumAll\! \expect  \normsq{\grad[k,t]_m}  \nonumber \\
    &\qquad + \frac{(\ubound\eta)^{2} L \tau}{4}\left(\frac{2}{M}+\lbound^{2} L \eta(\tau-1)\right) \sigma_{\varepsilon}^{2} \nonumber \\
    &\qquad + \frac{\eta (\ubound^2 - \lbound^2)}{M} \sum_{m=1}^M R^{(k)}_{m} 
         + \frac{L}{2M}\sigma_{k}^2 \label{eq:lr_choice} \\
    & \; \overset{\cirone}{\leqslant} - \frac{\lbound^2 \eta \tau}{2} \expect \normsq{\nabla f(\nw[k])} 
     + \frac{(\ubound\eta)^{2} L \tau}{4}\Big(\frac{2}{M} + \nonumber\\ & \qquad \lbound^{2} L \eta(\tau-1)\Big) \sigma_{\varepsilon}^{2} 
     \!+\! \frac{\eta (\ubound^2\!-\!\lbound^2)}{M} \sum_{m=1}^M R^{(k)}_{m} 
         \!+\! \frac{L}{2M}\sigma_{k}^2,
\end{align}
\end{subequations}
where $\cirone$ follows from the restrictions on learning rate. 
We rewrite the restriction as 
\begin{subequations}
\begin{align}
    & \eta  \leqslant \frac{-\frac{L\tau}{c_1^2} + \sqrt{\frac{L^2 \tau^2 }{c_1^4}+ \frac{2L^2 \tau (\tau - 1) }{c_2^2}}  }{L^2 \tau (\tau-1)} \\
    & \quad \overset{\cirone}{\leqslant} \frac{-1 + \sqrt{ 1+ \frac{2 c_1^4 }{c_2^2}}  }{L (\tau-1) c_1^2} 
     \overset{\cirtwo}{\leqslant} \frac{c_1^2}{L(\tau-1)c_2^2}, \label{eq:lr_restrict}
\end{align}
\end{subequations}
where in $\cirone$ we use $\tau(\tau - 1) \leqslant \tau^2$, and $\cirtwo$ holds due to the Bernoulli inequality $(1+x)^{\frac{1}{2}} \leqslant 1 + \frac{1}{2}x,\; \forall\; x\in[-1, \infty)$.  
Based on \eqref{eq:lr_restrict}, we set the learning rate $\eta = O\left((\frac{c_1}{c_2})^2 \frac{1}{L\tau \sqrt{K}}\right)$. 
The choice of the learning rate $\eta$ makes the upper bound in \eqref{eq:lr_choice} independent of the client local gradient $\grad[k,t]_m$.
Summing up over $K$ communication rounds yields 
\begin{subequations}
\begin{align}
    & \frac{1}{K} \sum_{k=0}^{K-1} \lbound^2 \expect \normsq{\nabla f(\nw[k])} 
     \leqslant  \frac{2 \left[f(\nw[0]) - f(\nw^{*}) \right] }{\eta\tau K} \nonumber \\
    & + \ubound^2 L \eta \left[ \frac{1}{M} +  \frac{\lbound^2 L\eta (\tau-1)}{2} \right] \sigma_{\varepsilon}^2 
      +  \frac{L}{\eta\tau K M} \sum_{k=0}^{K-1} \sigma_{k}^2  \nonumber \\
    & + \frac{2(\ubound^2 - \lbound^2)}{\tau MK} \sum_{k=0}^{K-1}\sum_{m=1}^{M} R^{(k)}_{m}.  
\end{align}
\end{subequations}
\end{proof}

\subsection{Proof of \paperlemma{lemma:global_nw_aggregation} \label{proof:global_nw_aggregation}}
\begin{proof}
From the reconstruction rule (\ref{eq:latent_weight_reconstruct}), we have 
\begin{equation}\label{eq:global_nw_reconstruct}
    \nw[k+1] = 2 \p[k+1] - 1.
\end{equation}
The $i$th entry of $\p[k]$ is defined in (\ref{eq:vote_distribution}), i.e., 
\begin{equation}\label{eq:ith_entry_pg}
    p^{(k+1)}_{i}= \frac{1}{M} \sum_{m=1}^{M} \ind{\widehat{w}^{(k,\tau)}_{m,i}=1 }. 
\end{equation}
Substituting (\ref{eq:ith_entry_pg}) into (\ref{eq:global_nw_reconstruct}) yields 
\begin{equation}
    w^{(k+1)}_{i}
    = \frac{1}{M} \sum_{m=1}^{M}\left( 2 \cdot \ind{\widehat{w}^{(k,\tau)}_{m,i}=1 } -1 \right).
\end{equation}
Note that 
\begin{equation}
    2 \cdot \ind{\widehat{w}^{(k,\tau)}_{m,i}=1 } - 1 = 
    \left\{\begin{array}{l @{\; \;} l}
    +1, & \widehat{w}_{m}^{(k, \tau)}=+1, \\[3pt]
    -1, & \widehat{w}_{m}^{(k, \tau)}=-1,
\end{array}\right.
\end{equation}
we have 
\begin{equation}
    w^{(k+1)}_{i}= \frac{1}{M} \sum_{m=1}^{M} \widehat{w}^{(k)}_{i},
\end{equation}
which completes the proof. 
\end{proof}

\subsection{Proof of \paperlemma{lemma:bound_w_divergence} \label{proof:bound_w_divergence}}
\begin{proof}
With the local initialization and update method described in \paperalg{alg:fedvote}, we have 
\begin{subequations}
\begin{align}
    &\; \expect{ \normsq{\nw[k,\tau]_{m}\!-\!\nw[k]}}  %
    = \expect{ \normsq{ \varphi(\lw[k,\tau]_{m})\!-\!\varphi(\lw[k,0]_{m}) }   } \\
    \overset{\cirone}{\leqslant} &\;  \ubound^2  \expect{ \normsq{\lw[k,\tau]_{m} - \lw[k,0]_{m} }} 
    = \ubound^2 \expect{ \normSQ{\sumtau  -\eta \, \stograd[k,t]}  }.
\end{align}
\end{subequations}
where $\cirone$ comes from the Lipschitz condition in \paperlemma{lemma:lipschitz_continuity}. 
By decomposing $\stograd[k,t]$ into $\grad[k,t]_m + \eps[k,t]_{m}$, we have 
\begin{subequations}
    \begin{align}
    &\; \expect{ \normsq{\nw[k,\tau]_{m}\!-\!\nw[k]}} \nonumber\\
    = &\; (\ubound\eta)^2  \Bigg( \expect \normSQ{\sumtau \grad[k,t]_m }
    +  \expect \normSQ{\sumtau \eps[k,t]_{m}} \Bigg) \\
    \leqslant &\;  (\ubound\eta)^2 \tau
        \left( \sumtau \expect \normsq{\grad[k,t]_m} 
        + \sigma_{\varepsilon}^{2} \right),
\end{align}
\end{subequations}
which completes the proof. 
\end{proof}

\subsection{Proof of \paperlemma{lemma:bound_grad_divergence} \label{proof:bound_grad_divergence}}
\begin{proof}
We have 
\begin{subequations}
\begin{align}
    &\; \expect{\innprod{\nabla f(\nw[k])}{ \de[k] }} \nonumber \\
    = &\; \expect \innprod{\nabla f(\nw[k])}
    { \nw[k]\!-\! \frac{1}{M}\!\sum_{m=1}^{M} \nw[k,\tau]_{m} \!+\! \frac{1}{M}\!\sum_{m=1}^{M}\!\quanterr[k]_{m}}  \\
    = &\; \expect \innprod{\nabla f(\nw[k])}{\frac{1}{M}\sum_{m=1}^{M} \varphi(\lw[k]) - \varphi(\lw[k,\tau]_{m}) } \label{eq:interm_inner_product}
\end{align}
\end{subequations}
From the mean value theorem, for $h^{(k)}_{i}, h^{(t,\tau)}_{m,i} \in \mathbb{R}$, there exists a point $c^{(k)}_{m,i} \in \interval$ such that 
\begin{equation}
    f^{\prime}(c^{(k)}_{m,i}) = \frac{\varphi(h^{(k)}_i) - \varphi(h^{(k,\tau)}_{m,i})}{h^{(k)}_i - h^{(k,\tau)}_{m,i}},
\end{equation}   
where $\interval$ is an open interval with endpoints $h^{(k)}_i$ and $h^{(k,\tau)}_{m,i}$, i.e.,
\begin{equation}
    \interval = \left\{
    \begin{array}{l @{\; \;} l}
        (h^{(k)}_i, h^{(k,\tau)}_{m,i}), & \textrm{ if } h^{(k)}_i < h^{(k,\tau)}_{m,i}, \\[3pt]
        (h^{(k,\tau)}_{m,i}, h^{(k)}_i), & \textrm{ otherwise. }
    \end{array}
    \right.
\end{equation}
Let  $C_{m}^{(k)} = \diag{\frac{\diff \varphi}{\diff c^{(k,\tau)}_{m,1}}, \cdots, \frac{\diff \varphi}{\diff c^{(k,\tau)}_{m,d}}}$, we have 
\begin{subequations}
\begin{align}
    \varphi(\lw[k]) - \varphi(\lw[k,\tau]_{m}) 
    & = C^{(k)}_m (\lw[k] - \lw[k,\tau]_{m}) \\ 
    & = \eta\, C^{(k)}_m \sum_{t=0}^{\tau-1} \left(\grad[k,t]_m + \eps[k,t]_{m}\right), \label{eq:expand_lw_difference}
\end{align}
\end{subequations}
Substituting (\ref{eq:expand_lw_difference}) into (\ref{eq:interm_inner_product}) yields
\begin{equation}
    \expect{\innprod{\nabla f(\nw[k])}{ \de[k] }}  
    = \frac{\eta}{M} \sum_{m=1}^{M}\sum_{t=0}^{\tau-1} \expect \innerprod{\nabla f(\nw[k])}{C_{m}^{(k)}\,\grad[k,t]}. \label{eq:inner_product_linear_approximation}
\end{equation}
According to the chain rule, $\grad[k,t]_m$ can be written as
\begin{equation}
    \grad[k,t]_m = \nabla_{\lw} f_{m}(\varphi(\lw[k,t]_{m})) = D_{m}^{(k,t)}\, \nabla f_m(\nw[k,t]_{m}), \label{eq:chain_rule}
\end{equation}
where $D_{m}^{(k,t)} = \diag{\frac{\diff \varphi}{\diff h^{(k,t)}_{m,1}}, \cdots, \frac{\diff \varphi}{\diff h^{(k,t)}_{m,d}}}$. 
To simplify the notations, let $B_{m}^{(k,t)} = C_{m}^{(k)} D_{m}^{(k,t)}$. Note that $B_{m}^{(k,t)}$ is still a diagonal matrix, where the $i$th diagonal element $b^{(k,t)}_{m,i}$ is
\begin{equation}
    b^{(k,t)}_{m,i} \triangleq (B_{m}^{(k,t)})_{i,i} = \frac{\diff \varphi}{\diff c^{(k,\tau)}_{m,i}} \cdot \frac{\diff \varphi}{\diff h^{(k,t)}_{m,i} }.
\end{equation}
Substituting (\ref{eq:chain_rule}) into (\ref{eq:inner_product_linear_approximation}) we have 
\begin{align}
    & \expect{\innprod{\nabla f(\nw[k])}{ \de[k] }} \nonumber \\
    = &  \frac{\eta}{M} \sum_{m=1}^{M}\sum_{t=0}^{\tau-1} \expect \innerprod{\nabla f(\nw[k])}{B_{m}^{(k,t)}\, \nabla f_{m}(\nw[k,t]_{m}) }. \label{eq:expected_inner_product}
\end{align}
We first focus on the following inner product:
\begin{align}
    &\; \innerprod{\nabla f(\nw[k])}{B_{m}^{(k,t)}\, \nabla f_{m}(\nw[k,t]_{m}) } \nonumber \\
    = &\; \sum_{i=1}^{d} (\nabla f(\nw[k]))_i \times b^{(k,t)}_{m,i} (\nabla f_{m}(\nw[k,t]_{m}))_i , \label{eq:expand_inner_product}
\end{align}
where $(\nabla f)_{i}$ denotes the $i$th entry of the gradient vector. 
We consider the set $\posindex$ defined as 
\begin{equation}
    \posindex \triangleq \left\{i \in [d] |(\nabla f(\nw[k]))_i \cdot (\nabla f_{m}(\nw[k,t]_{m}))_i \geqslant 0 \right\}, \\
\end{equation}
Let $\negindex$ denote the complement of $\posindex$. 
The result in (\ref{eq:expand_inner_product}) can be bounded as 
\begin{subequations}
\begin{align}
    & \; \expect{\innprod{\nabla f(\nw[k])}{B_{m}^{(k,t)}\, \nabla f_{m}(\nw[k,t]_{m}) }} \\
    \overset{\cirone}{\geqslant} & \; \lbound^2 \sum_{i\in\posindex} \expect \left[ (\nabla f(\nw[k]))_i \cdot (\nabla f_{m}(\nw[k,t]_{m}))_i \right] + \nonumber
    \\ &\;  \ubound^2  \sum_{i\in\negindex} \expect \left[ (\nabla f(\nw[k]))_i \cdot (\nabla f_{m}(\nw[k,t]_{m}))_i \right] \\
    = &\; \lbound^2 \, \expect \innerprod{\nabla f(\nw[k])}{\nabla f_{m}(\nw[k,t]_{m})} + \nonumber \\
    &\; (\ubound^2\!-\!\lbound^2)\hspace*{-4pt} \sum_{i\in\negindex}\hspace*{-4pt} \expect\!\left[(\nabla f(\nw[k]))_i \cdot (\nabla f_{m}(\nw[k,t]_{m}))_i \right],
\end{align}
\end{subequations}
where $\cirone$ follows from \paperassumption{assumption:normalization_function}. 
To simplify the notation, 
let $R^{(k)}_{m} $ denote the accumulated gradient divergence, namely, 
\begin{equation}
    R^{(k)}_{m} \triangleq - \sum_{t=0}^{\tau-1}\hspace*{-4pt} \sum_{i\in\negindex}\hspace*{-4pt} \expect\left[ (\nabla f(\nw[k]))_i  (\nabla f(\nw[k,t]_{m}))_i \right].
\end{equation}
The expected inner product in (\ref{eq:expected_inner_product}) can be bounded as 

\begin{align}
    & \expect{\innprod{\nabla f(\nw[k])}{ \de[k] }} 
    \geqslant \frac{\lbound^2 \eta}{M} \sum_{m=1}^M\sum_{t=0}^{\tau-1} \expect  \big\langle\nabla f(\nw[k]) , \nonumber \\
    &\qquad \nabla f_{m}(\nw[k,t]_{m}) \big\rangle     
    - \frac{\eta(\ubound^2 - \lbound^2)}{M} \sum_{m=1}^M R^{(k)}_{m}. \label{eq:expected_inner_product}
\end{align}
By rewriting the inner product term in \eqref{eq:expected_inner_product} according to $2\inprod{\vecA}{\vecB} = \normsq{\vecA} + \normsq{\vecB} - \normsq{\vecA - \vecB}$, we have 

\begin{subequations}
\begin{align}
    & \; \expect{\innprod{\nabla f(\nw[k])}{ \de[k] }} \nonumber \\
    \geqslant  &\;   \frac{\lbound^2\eta}{2M} \sumAll 
        \Big(\expect \normsq{\nabla f(\nw[k])}
        + \expect \normsq{\nabla f(\nw[k,t]_m)} \nonumber \\  
    &   \!-\! \expect \normsq{\nabla f(\nw[k,t]_m) \!-\! \nabla f(\nw[k]) } \Big) 
        \!-\! \frac{\eta(\ubound^2 \!-\! \lbound^2)}{M} \sum_{m=1}^M R^{(k)}_{m} \\
    \overset{\cirone}{\geqslant} &\; \frac{\lbound^2 \eta \tau}{2} \normsq{\nabla f(\nw[k])} 
    + \frac{\lbound^2 \eta}{2M} \sumAll \Big( \expect \normsq{\nabla f(\nw[k,t]_{m}) } \nonumber \\
    & - L^2  \expect \normsq{\nw[k,t]_{m} - \nw[k]} \Big) 
    -  \frac{\eta(\ubound^2 - \lbound^2)}{M} \sum_{m=1}^M R^{(k)}_{m}, \label{eq:grad_divergence_interm_result}
\end{align}
\end{subequations}
where $\cirone$ holds due to \paperassumption{assumption:Lsmooth} of Lipschitz smoothness for the objective function. 
From Lemma \ref{lemma:bound_w_divergence}, we can show that
\begin{equation}\label{nested_weight_divergence}
    \expect \normsq{\nw[k,t]_{m} - \nw[k]} \leqslant 
    (\ubound\eta)^2 t \left( \sum_{n=0}^{t-1} \normsq{ \grad[k,n] } + \sigma_{\varepsilon}^2 \right),
\end{equation}
where $t=1, \, \cdots,\, \tau-1$. 
Substituting (\ref{nested_weight_divergence}) 
in (\ref{eq:grad_divergence_interm_result}) yields
\begin{subequations}
    \begin{align}
        & \expect{\innprod{\nabla f(\nw[k])}{ \de^{(k)} }} \nonumber \\
        \geqslant &\;  \frac{\lbound^2 \eta \tau}{2} \expect \normsq{\nabla f(\nw[k])} + 
        \frac{\lbound^2 \eta}{2M} \sumAll \expect \normsq{ \nabla f(\nw[k,t]_m) } \nonumber \\ 
            & \!-\! \frac{(\lbound\ubound L)^2 \eta^3}{4M} \sumM \sum_{n=0}^{\tau-1} \left( \tau(\tau-1) \!-\! n(n+1) \right) \expect \normsq{\grad[k,n]}   \nonumber \\
            & - \frac{ (\lbound\ubound L)^2 \eta^3 \tau(\tau-1)}{4} \sigma_{\varepsilon}^2 
            - \frac{\eta (\ubound^2 - \lbound^2)}{M} \sum_{m=1}^M R^{(k)}_{m} \\
        \overset{\cirone}{\geqslant} &\; \frac{\lbound^2 \eta \tau}{2} \expect \normsq{\nabla f(\nw[k])} 
                +   \frac{\lbound^2 \eta}{2M} \sumAll \expect \normsq{ \nabla f(\nw[k,t]_m) } \nonumber \\
            & - \frac{\eta}{4M} (\lbound \ubound L)^2\eta^2\tau(\tau-1) \sumAll \expect  \normsq{\grad[k,t]_m}  \nonumber \\
            &  -\frac{ (\lbound \ubound L)^2\eta^3 \tau(\tau-1)}{4} \sigma_{\varepsilon}^2 
            - \frac{\eta (\ubound^2 - \lbound^2)}{M} \sum_{m=1}^M R^{(k)}_{m} \\
        \overset{\cirtwo}{\geqslant} &\; \frac{\lbound^2 \eta \tau}{2} \expect \normsq{\nabla f(\nw[k])}  \nonumber \\
            & + \frac{\lbound^2\eta}{4M} \left( 2 - (\ubound L)^2\eta^2\tau(\tau-1) \right)\sumAll \expect  \normsq{\grad[k,t]_m}  \nonumber \\
            & -\frac{(\lbound \ubound L)^2\eta^3 \tau(\tau-1)}{4} \sigma_{\varepsilon}^2 
            \!-\! \frac{\eta (\ubound^2 - \lbound^2)}{M} \sum_{m=1}^M R^{(k)}_{m}      
    \end{align}
\end{subequations}
where $\cirone$ follows from $\tau(\tau-1) - n(n+1) \leqslant \tau(\tau-1)$; $\cirtwo$ holds due to the chain rule in (\ref{eq:chain_rule}) and \paperassumption{assumption:normalization_function}.

\end{proof}

\subsection{Proof of \paperlemma{lemma:bound_globalw_diff}}\label{proof:bound_globalw_diff}
\begin{proof}
We have 
\begin{subequations}
\begin{align}
    & \expect\!\normsq{\de[k]}
    \!=\!\expect\!\normSQ{\nw[k]\!-\!\frac{1}{M}\!\sum_{m=1}^{M}\!\nw[k,\tau]_{m}\!+\!\frac{1}{M}\!\sum_{m=1}^{M}\!\quanterr[k]_{m}}\\
    & = \expect\!\normSQ{\frac{1}{M}\sum_{m=1}^{M} \varphi(\lw[k])\!-\!\varphi(\lw[k,\tau]_{m}) + \frac{1}{M} \sum_{m=1}^{M} \quanterr[k]_{m}} \\
    & = \expect \normSQ{\frac{1}{M}\sum_{m=1}^{M} \varphi(\lw[k]) - \varphi(\lw[k,\tau]_{m})}
      \!+\!\expect \normSQ{\frac{1}{M}\sumM \quanterr[k]_{m} }  \nonumber \\
        &  + 2\expect \innerprod{\frac{1}{M}\sum_{m=1}^{M} \varphi(\lw[k]) - \varphi(\lw[k,\tau]_{m})}
        { \frac{1}{M}\sumM \quanterr[t]_{m} }  \\
    \overset{\cirone}{=} &\; \expect\!\normSQ{\frac{1}{M}\sum_{m=1}^{M} \varphi(\lw[k])\!-\!\varphi(\lw[k,\tau]_{m})}
        \!+\! \expect \normSQ{\frac{1}{M}\sumM \quanterr[k]_{m} } \label{eq:interm_bound}.  
\end{align}
\end{subequations}
where $\cirone$ comes from \paperassumption{assumption:quant}. 
For the first term in \eqref{eq:interm_bound},
\begin{subequations}
\begin{align}
    &\; \expect \normSQ{\frac{1}{M} \sum_{m=1}^M (\varphi(\lw[k]) - \varphi(\lw[k,\tau]_{m}) } \nonumber \\
    \leqslant  &\; \frac{1}{M} \sum_{m=1}^M \expect \normSQ{\varphi(\lw[k]) - \varphi(\lw[k,\tau]_{m}) } \\
    \overset{\cirone}{\leqslant}  &\; \frac{\ubound^2}{M} \sum_{m=1}^M \normsq{\lw[k] - \lw[k,\tau]_{m}} \\
    \overset{\cirtwo}{=} &\; \frac{(c_2\eta)^2}{M} \sum_{m=1}^M \normSQ{-\eta \sum_{t=0}^{\tau-1} \stograd[k,t] } \\
    \overset{\cirthree}{\leqslant} &\; \frac{(c_2\eta)^2\tau}{M} \sum_{m=1}^M \sum_{t=0}^{\tau-1} \expect \normSQ{\grad[k,t]_m} +  \frac{(c_2\eta)^2 \tau}{M} \sigma_{\varepsilon}^2, \label{eq:term1}
\end{align}
\end{subequations}
where $\cirone$ is due to \paperlemma{lemma:lipschitz_continuity}
$\cirthree$ follows from \paperassumption{assumption:sto_grad}. 
For the second term in \eqref{eq:interm_bound} we have
\begin{equation}\label{eq:term2}
    \expect \normSQ{\frac{1}{M}\sumM \quanterr[k]_{m} } \leqslant  \frac{1}{M} \sigma_{\zeta}^2. 
\end{equation}
Combing the results of \eqref{eq:term1} and \eqref{eq:term2} completes the proof. 

\end{proof}

\bibliographystyle{IEEEtran}
\bibliography{ref}

\end{document}